\documentclass{article}

\PassOptionsToPackage{numbers, compress}{natbib}
\usepackage[final]{neurips_2024}

\usepackage[utf8]{inputenc} %
\usepackage[T1]{fontenc}    %
\usepackage[colorlinks=true,citecolor=black,urlcolor=black,linkcolor=black]{hyperref}       %
\usepackage{url}            %
\usepackage{booktabs}       %
\usepackage{amsfonts}       %
\usepackage{nicefrac}       %
\usepackage{microtype}      %
\usepackage{xcolor}         %
\usepackage[english]{babel}

\usepackage{graphicx}
\usepackage{subfigure}
\usepackage{multirow}
\usepackage{subcaption}
\usepackage{float}

\usepackage{amsmath}
\usepackage{mathrsfs}
\usepackage{amssymb}
\usepackage{mathtools}
\usepackage{amsthm}
\usepackage{trimclip} %
\usepackage{nicefrac} %
\usepackage{cancel} 
\usepackage{wrapfig}
\usepackage{tabularx}

\newcommand{\sn}[2]{\ensuremath{{#1}\times 10^{#2}}}

\usepackage[capitalize,noabbrev]{cleveref}

\crefname{section}{Sec.}{Secs.}
\crefname{proposition}{Prop.}{Props.}
\crefname{lemma}{Lem.}{Lems.}
\crefname{model}{Mod.}{Mods.}
\crefname{appendix}{App.}{Apps.}
\crefname{algorithm}{Alg.}{Algs.}
\crefname{equation}{Eqn.}{Eqns.}
\crefname{figure}{Fig.}{Figs.}
\crefname{example}{Ex.}{Exs.}

\theoremstyle{plain}
\newtheorem{theorem}{Theorem}[section]

\theoremstyle{definition}

\theoremstyle{remark}

\theoremstyle{definition}
\newtheorem{example}{Example}[section]

\usepackage[textsize=tiny]{todonotes}

\newcommand{\mathdefault}[1][]{} %

\usepackage{layouts} %
\usepackage{tikz,pgfplots,pgfmath}

\newcommand\inputpgf[2]{{
	\let\pgfimageWithoutPath\pgfimage
		\renewcommand{\pgfimage}[2][]{\pgfimageWithoutPath[##1]{#1/##2}}
	\let\includegraphicsWithoutPath\includegraphics
		\renewcommand{\includegraphics}[2][]{\includegraphicsWithoutPath[##1]{#1/##2}}
\input{#1/#2}
}}

\newcommand{\resizeinpgf}[3][\textwidth]{
    \resizebox{#1}{!}{\input{#3/#3}}
}

\usepackage[nolist,smaller]{acronym}
\usepackage{xspace}
\usepackage{xcolor}

\definecolor{black}{HTML}{000000}
\definecolor{orange}{HTML}{E69F00}
\definecolor{blue}{HTML}{56B4E9}
\definecolor{green}{HTML}{009E73}
\definecolor{dark_blue}{HTML}{0072B2}
\definecolor{dark_orange}{HTML}{D55E00}
\definecolor{pink}{HTML}{CC79A7}
\definecolor{white}{HTML}{111111}
\definecolor{grey}{HTML}{808080} %
\usepackage{bm}

\newcommand{\mbf}[1]{\mathbf{#1}}

\newcommand{\mtrian}{{\raisebox{-0.25ex}{\clipbox{0em 1.25ex 0em 0em}{$\triangleq$}}}}
\newcommand{\defeq}{\overset{\mtrian}{=}}

\newcommand{\permeq}{\widetilde{=}} %

\newcommand{\KL}{\mathcal{KL}}
\newcommand{\LL}{\mathcal{L}}

\newcommand{\vu}{\mbf{u}}
\newcommand{\vf}{\mbf{f}}
\newcommand{\vx}{\mbf{x}}
\newcommand{\vy}{\mbf{y}}

\newcommand{\vg}{\mbf{g}}
\newcommand{\vm}{\mbf{m}}

\newcommand{\vv}{\mbf{v}}
\newcommand{\vr}{\mbf{r}}

\newcommand{\MA}{\mbf{A}}

\newcommand{\MJ}{\mbf{J}}
\newcommand{\MP}{\mbf{P}}
\newcommand{\MV}{\mbf{V}}
\newcommand{\MY}{\mbf{Y}}
\newcommand{\MX}{\mbf{X}}
\newcommand{\MZ}{\mbf{Z}}
\newcommand{\MW}{\mbf{W}}
\newcommand{\MU}{\mbf{U}}
\newcommand{\MH}{\mbf{H}}
\newcommand{\MS}{\mbf{S}}

\newcommand{\MK}{\mbf{K}}
\newcommand{\MI}{\mbf{I}}

\newcommand{\N}{\mathrm{N}}
\newcommand{\T}{\top}    %
\newcommand{\R}{\mathbb{R}}    %
\newcommand{\BO}{\mathcal{O}}    %

\newcommand{\GP}{\mathcal{GP}}
\newcommand{\E}{\mathbb{E}}    %

\newcommand{\MF}{\mbf{F}}

\newcommand{\blkdiag}{\text{blkdiag}}

\DeclareMathOperator*{\argmax}{arg\,max}

\newcommand{\kron}{\raisebox{1pt}{\ensuremath{\:\otimes\:}}} %

\newcommand{\dif}[2]{\mathrm{\frac{\mathop{\mathrm{d}#1}}{\mathop{\mathrm{d}#2}}}}

\newcommand{\diff}[2]{\mathrm{\frac{\partial\mathit{#1}}{\partial\mathit{#2}}}}
\newcommand{\diffII}[2]{\mathrm{\frac{d^2\mathit{#1}}{d\mathit{#2}^2}}}
\newcommand{\pdiffII}[2]{\mathrm{\frac{\partial^2\mathit{#1}}{\partial\mathit{#2}^2}}}

\newcommand{\mdiffII}[2]{\mathrm{\frac{d^2\mathit{#1}}{d^2\mathit{#2}}}}
\newcommand{\pmdiffII}[3]{\mathrm{\frac{\partial^2\mathit{#1}}{\partial\mathit{#2} \, \partial\mathit{#3}}}}

\newcommand{\hess}[2]{\mathrm{\frac{\partial^2\mathit{#1}}{\partial\mathit{#2} \, \partial\mathit{#2}^\T}}}

\newcommand{\vectext}{\text{vec}}
\newcommand{\Tr}{\text{Tr}}

\newcommand{\NN}{ {\cal N} }

\renewcommand{\mid}{\,|\,}

\newcommand{\nE}[2]{\E_{\, #1} \left[\, #2 \,\right]}
\newcommand{\nEtight}[2]{\E_{#1} \left[#2 \right]}

\newcommand{\nN}[3]{\N \left(\, #1 \, \mid \, #2, \, #3 \,\right)}
\newcommand{\nGauss}[2]{\N \left(\, #1, \, #2 \,\right)}

\newcommand{\nSin}[1]{\sin \left(\, #1 \, \right)}

\newcommand{\nKL}[2]{\KL \left[\, #1 \, ||\, #2 \,\right]}
\newcommand{\nBO}[1]{\BO \left( #1 \right)}

\newcommand{\nVec}[1]{\vectext (\, #1 \,)}
\newcommand{\nBlkdiag}[1]{\blkdiag \left(\, #1 \,\right)}
\newcommand{\nTr}[1]{\Tr \left[\, #1 \,\right]}

\newcommand{\acro}[1]{\textsc{#1}\xspace}

\newcommand{\RMSE}{\acro{rmse}}
\newcommand{\NLPD}{\acro{nlpd}}
\newcommand{\CRPS}{\acro{crps}}

\newcommand{\gp}{\acro{gp}}
\newcommand{\gps}{\textsc{gp}s\xspace}
\newcommand{\CVI}{\acro{cvi}}

\newcommand{\SVGP}{\acro{svgp}}
\newcommand{\LTISDE}{\acro{lti-sde}}

\newcommand{\EKS}{\acro{eks}}
\newcommand{\PMM}{\acro{pmm}}

\newcommand{\VI}{\acro{vi}}

\newcommand{\X}{\MX} %
\newcommand{\gridX}{\X^{(\mathrm{st})}} %
\newcommand{\Xs}{\X_{\SpaceIdx}} %
\newcommand{\Xt}{\X_{\TimeIdx}} %

\newcommand{\NumLatents}{\mathrm{Q}} %
\newcommand{\NumTasks}{\mathrm{P}} %
\newcommand{\NumData}{N} %
\newcommand{\SpaceIdx}{\mathrm{s}}
\newcommand{\TimeIdx}{\mathrm{t}}
\newcommand{\Ns}{\NumData_{\SpaceIdx}} %
\newcommand{\Nt}{\NumData_{\TimeIdx}} %

\newcommand{\ld}{\text{ld}} %
\newcommand{\dl}{\text{dl}} %

\newcommand{\D}{F}

\newcommand{\Ms}{M_\Space}

\newcommand{\HELMHOLTZ}{\acro{helmholtz-gp}} %

\newcommand{\Q}{Q} %
\renewcommand{\P}{P} %
\newcommand{\KLD}{\acro{kl}}
\newcommand{\ELL}{\acro{ell}}

\newcommand{\ELBO}{\acro{elbo}}

\newcommand{\RBF}{RBF\xspace}
\newcommand{\Matern}{Mat\'ern\xspace}

\newcommand{\MaternThreeTwo}{\Matern-\nicefrac{3}{2}\xspace}

\newcommand{\MaternSevenTwo}{\Matern-\nicefrac{7}{2}\xspace}

\newcommand{\ie}{\textit{i.e.}\xspace}

\newcommand{\wrt}{\textit{w.r.t.}\xspace}

\newcommand{\psd}{\textit{p.s.d}\xspace}

\newcommand{\RomanNumeral}[1]{\uppercase\expandafter{\romannumeral #1\relax}}

\newcommand{\AUTOIP}{\acro{autoip}}
\newcommand{\svgp}{\acro{svgp}}
\newcommand{\vgp}{\acro{vgp}}

\newcommand{\PhySS}{\acro{physs}}

\newcommand{\PIGP}{\acro{physs-gp}}
\newcommand{\vPIGP}{\PhySS-\vgp} %
\newcommand{\sPIGP}{\PhySS-\svgp} %
\newcommand{\hPIGP}{\PhySS-$\acro{vgp}_\acro{h}$\xspace} %
\newcommand{\hsPIGP}{\PhySS-$\acro{svgp}_\acro{h}$\xspace} %
\newcommand{\ekfPIGP}{\PhySS-\EKS} %
\newcommand{\StateIdx}{k}
\newcommand{\Nstate}{d}

\newcommand{\IWP}{\acro{iwp}}

\newcommand{\PDE}{\acro{PDE}}
\newcommand{\DE}{\acro{DE}}
\newcommand{\ODE}{\acro{ODE}}

\newcommand{\PDEs}{\textsc{PDE}s\xspace}
\newcommand{\DEs}{\textsc{DE}s\xspace}
\newcommand{\ODEs}{\textsc{ODE}s\xspace}

\newcommand{\param}{\boldsymbol\xi}
\newcommand{\natp}{\boldsymbol\lambda}
\newcommand{\apxnatp}{\boldsymbol{\widetilde{\lambda}}}
\newcommand{\priornatp}{\boldsymbol\eta}
\newcommand{\meanp}{\boldsymbol\mu}
\newcommand{\apxmeanp}{\widetilde{\boldsymbol\mu}}

\newcommand{\hyperparam}{\boldsymbol{\theta}}

\newcommand{\apxy}{\widetilde{\MY}}
\newcommand{\apxv}{\widetilde{\MV}}

\newcommand{\Time}{t}
\newcommand{\Space}{s}

\newcommand{\fproc}{f}
\newcommand{\f}{\vf}

\newcommand{\StateFProc}{\bar{f}}
\newcommand{\StateF}{\bar{\vf}}

\newcommand{\StateU}{\bar{\vu}}

\newcommand{\TransF}{\F}

\newcommand{\F}{\MF}

\newcommand{\DiffOp}{\mathcal{D}}
\newcommand{\DiffOpTilde}{\tilde{\mathcal{D}}}

\newcommand{\Diff}[1]{\DiffOp \, {#1}}
\newcommand{\TDiff}[1]{{#1} \, \DiffOp^{*}} %

\newcommand{\Diffd}[2]{\DiffOp_#1 \, {#2}}
\newcommand{\TDiffd}[2]{{#2} \, \DiffOp_{#1}^{*}} %

\newcommand{\DiffS}[1]{\Diffd{\SpaceIdx}{#1}}
\newcommand{\DiffT}[1]{\Diffd{\TimeIdx}{}{#1}}

\newcommand{\TDiffT}[1]{\TDiffd{\TimeIdx}{#1}}

\newcommand{\DiffKt}[2]{\MK^{\DiffOp}_\Time({#1}_\TimeIdx, {#2}_\TimeIdx)}
\newcommand{\DiffKs}[2]{\MK^{\DiffOp}_\Space({#1}_\SpaceIdx, {#2}_\SpaceIdx)}

\newcommand{\subDiffKs}[2]{\widetilde{\MK}^{\DiffOp}_\Space({#1}_\SpaceIdx, {#2}_\SpaceIdx)}

\newcommand{\KState}{\bar{\MK}}

\newcommand{\KT}{\MK_\Time}
\newcommand{\KS}{\MK_\Space}

\newcommand{\Nd}{D}
\newcommand{\Ndt}{d_\Time}
\newcommand{\Nds}{d_\Space}
\newcommand{\Nc}{C} %

\newcommand{\Zero}{\mathbf{0}}
\usetikzlibrary{calc,decorations.pathmorphing,patterns}
\pgfdeclaredecoration{penciline}{initial}{
    \state{initial}[width=+\pgfdecoratedinputsegmentremainingdistance,
    auto corner on length=1mm,]{
        \pgfpathcurveto%
        {%
            \pgfqpoint{\pgfdecoratedinputsegmentremainingdistance}
                      {\pgfdecorationsegmentamplitude}
        }
        {%
        \pgfmathrand
        \pgfpointadd{\pgfqpoint{\pgfdecoratedinputsegmentremainingdistance}{0pt}}
                    {\pgfqpoint{-\pgfdecorationsegmentaspect
                     \pgfdecoratedinputsegmentremainingdistance}%
                               {\pgfmathresult\pgfdecorationsegmentamplitude}
                    }
        }
        {%
        \pgfpointadd{\pgfpointdecoratedinputsegmentlast}{\pgfpoint{1pt}{1pt}}
        }
    }
    \state{final}{}
}

\newcommand{\Eye}{{\bm{I}}}
\newcommand{\TransW}{\MixingMatrixKron \, \StackedGram{\MX}{\MX} \, \MixingMatrixKron^\T}
\newcommand{\TransS}{\MixingMatrixKron \, \MS \, \MixingMatrixKron^\T}
\newcommand{\TransPhiInv}{\MixingMatrixKron \, \Phi^{-1} \, \MixingMatrixKron^\T}
\newcommand{\Stacked}[1]{\widetilde{#1}}
\newcommand{\StackedF}{\Stacked{\vf}}
\newcommand{\StackedGram}[2]{\Stacked{\MK}_{#1,#2}}
\newcommand{\MixingMatrix}{\MW}
\newcommand{\MixingMatrixKron}{\Stacked{\MW}}

\newcommand{\MultiOutputF}{\MF}

\title{Physics-Informed Variational State-Space \\ Gaussian Processes}

\author{%
  Oliver Hamelijnck\\
  University of Warwick\\
  \texttt{oliver.hamelijnck@warwick.ac.uk}
  \AND
  Arno Solin\\
  Aalto University\\
  \texttt{arno.solin@aalto.fi}\\
  \hphantom{\texttt{t.damoulas@warwick.ac.uk}}\\[-1em] %
  \And
  Theodoros Damoulas\\
  University of Warwick\\
  \texttt{t.damoulas@warwick.ac.uk}
}

\begin{document}

\maketitle
 
\begin{abstract}
  Differential equations are important mechanistic models that are integral to many scientific and engineering applications. With the abundance of available data there has been a growing interest in data-driven physics-informed models. Gaussian processes (GPs) are particularly suited to this task as they can model complex, non-linear phenomena whilst incorporating prior knowledge and quantifying uncertainty. Current approaches have found some success but are limited as they either achieve poor computational scalings or focus only on the temporal setting. This work addresses these issues by introducing a variational spatio-temporal state-space GP that handles linear and non-linear physical constraints while achieving efficient linear-in-time computation costs. We demonstrate our methods in a range of synthetic and real-world settings and outperform the current state-of-the-art in both predictive and computational performance.
\end{abstract}

\setlength{\columnsep}{8pt}
\setlength{\intextsep}{6pt}
\begin{wrapfigure}{r}{.33\textwidth}
  \vspace*{-2em}
  \centering\scriptsize

  \textbf{Na\"ive vs.\ Physics-Informed State-Space \gp (\PIGP)}\\[1em]
  \begin{tikzpicture}[decoration=penciline, decorate]

    \def\scale{2em}

    \definecolor{white}{HTML}{ffffff}
    \tikzstyle{blob}=[fill=blue,inner sep=2pt,circle,opacity=.5]      
    \tikzstyle{label}=[font=\scriptsize]  
    \tikzstyle{arr}=[->,-latex,line width=1pt]  
    \tikzstyle{line}=[<->,line width=1pt,draw=black!80]  
    \tikzstyle{line2}=[line width=1pt,draw=black!80]  

    \foreach \x in {0,1,2,4,5,6}
      \foreach \y in {0,...,6}
        \node[blob] (\x-\y) at (\scale*\x,\scale*\y) {};

    \foreach \x in {4,...,6}
      \draw[draw=blue,line width=2pt] (\x*\scale,0) -- (\x*\scale,6*\scale);

    \foreach \opa in {0.1,0.3,...,1.0}
       \node[fill=blue,opacity=\opa,text=black,rounded corners=3pt,inner sep=2em*(1-\opa)] at (3*\scale,4*\scale) {$\diff{f}{t} - \NN_\theta \, f = 0$};

    \draw[line] (0-2) -- (2-0);
    \draw[line] (0-0) -- (2-2);
    \draw[line] (0-1) -- (2-1);
    \draw[line] (1-0) -- (1-2);

    \draw[line2] (4-0) -- (4-2);
    \draw[line2,->] (4-1.center) -- (6-1);
    
    \node at (3*\scale,1*\scale) {vs.};

    \draw[arr] (-1*\scale,0) -- node[midway,rotate=90,fill=white]{Input, $x$} (-1*\scale,6*\scale);
    \draw[arr] (0,-1*\scale) -- node[midway,fill=white]{Time, $t$} (6*\scale,-1*\scale);
        
  \end{tikzpicture}
  \caption{The state-space formalism allows for linear-time inference in the temporal dimension.}
  \vspace*{-8pt}
  \label{fig:teaser}
\end{wrapfigure}
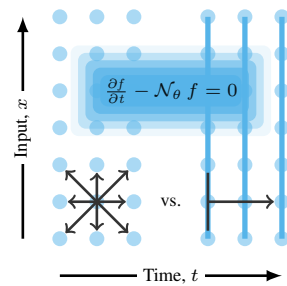
\section{Introduction}
Physical modelling is integral in modern science and engineering with applications from climate modelling \citep{stocker_introduction_to_climate_modelling:2011} to options pricing \citep{black_options:1973}. Here, the key formalism to inject mechanistic physical knowledge are differential equations (\DEs), which given initial and/or boundary values, are typically solved numerically \citep{borthwick2017introduction}. In contrast machine learning is data-driven, and aims to learn latent functions from observations. However the increasing availability of data has spurred interest in combining these traditional mechanistic models with data-driven methods through physics-informed machine learning. These hybrids approaches aim to improve predictive accuracy, computational efficiency by leveraging both physical inductive biases with observations \cite{karniadakis2021physics, pan_dce:2021}.

A principled way to incorporate prior physical knowledge is through Gaussian processes (\gps). \gps are stochastic processes and are a data-centric approach that facilitates the quantification of uncertainty. Recently \AUTOIP was proposed in order to integrate non-linear physics into \gps \citep{long_autoip:2022}, where solutions to ordinary and partial differential equations (\ODEs, \PDEs) are observed at a finite set of collocation points. This is an extension of the probabilistic meshless method (\PMM, \citep{cockayne_meshless_pde:2017}) to the variational setting such that non linear equations can be incorporated. Similarly,  \citep{berlinghieri_helmholtz_gp:2023} introduced \HELMHOLTZ, that constructs \gp priors that adhere to curl and divergence-free constraints. Such properties are required for the successful modelling of electromagnetic fields \citep{solin_ambient_magnet:2018} and ocean currents through the Helmholtz decomposition \citep{berlinghieri_helmholtz_gp:2023}. These approaches enable the incorporation of physics but incur a cubic computational complexity from needlessly computing full covariance matrices, as illustrated in \cref{fig:teaser}. For \ODEs (time-series setting), extended Kalman smoothers incorporate non-linear physics (\EKS)~\citep{tronarp_prob_ode_new_perspective:2018,kramer_stable_implementation:2024} and recover popular \ODE solvers whilst achieving linear-in-time complexity through state-space \gps \citep{schober_ode_runge_kutta:2014,hennig_probabilistic:2022}. 

In this work we propose a unified physics informed state-space \gp (\PIGP) that is a probabilistic models where mechanistic/physics knowledge is incorporated as an inductive bias. We can handle both linear and non-linear \PDEs and \ODEs whilst maintaining linear-in-time computational efficiency. We additionally derive a state-space variational inference algorithm that further reduces the computational cost in the spatial dimension. We recover \EKS, \PMM, and \HELMHOLTZ as special cases, and outperform \AUTOIP in terms of computational efficiency and predictive performance. In summary:

\begin{enumerate}
	\item We derive a state-space \gp that can handle spatio-temporal derivatives with a computational complexity that is  linear in the temporal dimension. 
	\item With this we derive a unifying state-space variational inference framework that allows the incorporation of both linear and non-linear PDEs whilst achieving a linear-in-time complexity and recovering state-of-the-art methods such as \EKS, \PMM and \HELMHOLTZ.
	\item We further explore three approximations, namely a structured variational posterior, spatial sparsity, and spatial minibatching, that reduce the cubic spatial computational costs to linear.
	\item We showcase our methods on a variety of synthetic and real-world experiments and outperform the current state-of-the-art methods \AUTOIP and \HELMHOLTZ both in terms computational and predictive performance. 
\end{enumerate}
Code to reproduce experiments is available at \url{https://github.com/ohamelijnck/physs_gp}.

\section{Background on Gaussian Processes}

\paragraph{Gaussian processes}
A GP is a distribution on an infinite collection of random variables such that any finite subset is jointly Gaussian \citep{rasmussen_gpml:2006}. Given observations $\MX \in \R^{N \times \D}$ and $\vy \in \R^{N}$ then
\begin{equation}
	p(\vy, f \mid \hyperparam) = \textstyle\prod^N_n p(y_n \mid f(\vx_n), \hyperparam) \, p(f \mid \hyperparam)
\end{equation}
is a joint model where $p(f \mid \hyperparam)$ is a zero mean GP prior with kernel $\MK(\cdot, \cdot)$, $f(\MX) \sim  p(f(\MX) \mid 0, \MK(\MX, \MX))$, and $\hyperparam$ are (hyper) parameters. We are primarily concerned with the spatio-temporal setting where we observe $\Nt$ temporal and $\Ns$ spatial observations $x_{\TimeIdx, \SpaceIdx} \in \R, \, y_{\TimeIdx, \SpaceIdx} \in \R$ on a spatio-temporal grid. Under a Gaussian likelihood, all quantities for inference and training are available analytically and, na\"ively, carry a dominant computational cost of $\BO((\Nt \, \Ns)^3)$.
For time series data, an efficient way to construct a GP over $f$ (and its time derivatives) is through the state-space representation of GPs. Given a Markov kernel, the temporal GP prior can be written as the solution of a discretised linear time-invariant stochastic differential equation (\LTISDE), which at time $k$ is
\begin{equation}
	\StateF_{\StateIdx+1} = \MA \, \StateF_\StateIdx + q_\StateIdx \quad \text{and} \quad y_\StateIdx \mid \StateF_\StateIdx \sim p(y_\StateIdx \mid \MH \, \StateF_\StateIdx),
\label{eqn:state_space_gp_form}
\end{equation}
where $\MA$ is a transition matrix, $q_\StateIdx$ is Gaussian noise, $\MH$ is an observation matrix, and $\StateF$ is a $\Nstate$-dimensional vector of temporal derivatives $\StateFProc = [\fproc(\cdot), \diff{\fproc(\cdot)}{x}, \pdiffII{\fproc(\cdot)}{x}, \cdots]^\T$. With appropriately designed states, matrices and densities, SDEs of this form represent a large class of \gp models, and Kalman smoothing enables inference in $\BO(\Nt \, \Nstate^3)$, see \citep{sarkka2019applied}. In the spatio-temporal setting, when the kernel matrix decomposes as a Kronecker product $\MK = \MK_t \kron \MK_s$, then with a Markov time kernel, a state space form is admitted. This takes a particularly convenient form where the state is $\StateF_t = [\StateFProc((\MX_s)_1, t), \cdots, \StateFProc((\MX_s)_{Ns}, t)]^\T$, 
and inference requires $\BO(\Nt (\Ns \, \Nstate)^3)$, see \citep{solin_thesis:2016}. 

\paragraph{Derivative Gaussian processes}
One main appeal of \gps is that they are closed under linear operators. Let  $\Diff{[\cdot]} = \DiffT{\DiffS{[\cdot]}}$ be linear functional that computes $\Nd = \Ndt \, \Nds $ space-time derivatives with $\DiffT{[\cdot]} = \left[ \cdot, \diff{\cdot}{t},  \pdiffII{\cdot}{t}, \cdots \right]$ and $\DiffS{[\cdot]}=\left[ \cdot, \diff{\cdot}{s},  \pdiffII{\cdot}{s}, \cdots \right]$, then at a finite set of index points, the joint prior between $\f$ and its time and spatial derivatives is
\begin{equation}
	p(\StateFProc(\MX))= \nN{
		\Diff{\vf}
	}{\Zero}{\TDiff{\Diff{\MK(\MX, \MX)}}}	
\label{eqn:background_derivative_gp}
\end{equation}
where $\StateFProc(\MX) = \Diff{\fproc(\MX)}$ and $\Diff{}^{*}$ is the adjoint of $\Diff{}$, meaning it operates on the second argument of the kernel \citep{sarkka_spde:2011}. When jointly modelling a single time and space derivative ($\Ndt=\Nds=1$) the latent functions are $\StateF = [\vf, \diff{\f}{\SpaceIdx}, \diff{\f}{\TimeIdx}, \pmdiffII{\f}{\TimeIdx}{\SpaceIdx}]^\T$ and the kernel is
\begin{equation*}
	\bar{\MK} = \TDiff{\Diff{\MK(\MX, \MX)}} = \left[\begin{matrix}
 		\MK & - & - & - \\	
 		\diff{}{s} \, \MK & \diff{}{s} \, \MK \, \diff{}{s}^\T & - & - \\	
 		\diff{}{t} \, \MK & \diff{}{t} \, \MK \, \diff{}{s}^\T & \diff{}{t} \, \MK \, \diff{}{t}^\T & - \\	
 		\pmdiffII{}{t}{s} \, \MK & \pmdiffII{}{t}{s} \, \MK \, \diff{}{s}^\T & \pmdiffII{}{t}{s} \, \MK \, \diff{}{t}^\T & \pmdiffII{}{t}{s} \, \MK \pmdiffII{}{t}{s}^\T
	\end{matrix} \right].
\end{equation*}
This is a multi-output prior whose samples are paths of $f$ with its corresponding derivatives. This prior is commonly known as a derivative \gp and has found applications in monotonic GPs \citep{riihimaki_monotonic_gp:2010}, input-dependent noise \citep{mchutchon_input_noise:2011,villacampa_noisy_input_gps:2021} and explicitly modelling derivatives \citep{solak_diff_gp:2002,eriksson_scaling_diff_gps:2018,padidar_scaling_diff_info:2021}. State-space \gps can be employed in the temporal setting since the underlying state computes $f(\vx)$ with its corresponding time derivatives.  In \cref{sec:st_diff_obs}, we extend this to the spatio-temporal setting.

\section{Physics-Informed State-Space Gaussian Processes (\PIGP)}
\label{sec:phi_gp_methods}

We now propose a flexible generative model for incorporating information from both data observations and (non-linear) physical mechanics. We consider general non-linear evolution equations of the form
\begin{equation}
	g(\NN_\theta \, f) = \diff{f}{t} - \NN_\theta \, f = 0 
\label{eqn:phi_gp_general_pde}
\end{equation}
with appropriate boundary conditions, where $f: \R^\D \to \R$ is the latent quantity of interest and $\NN_\theta$ is a non-linear differential operator \citep{raissi_pinn:2019}. We assume that $g: \R^{\P \cdot \Nd} \to \R$ is measurable, and is well-defined such that there are sensible solutions to the differential equation \cite{hennig_probabilistic:2022}. We wish to place a \gp prior over $f$ and update our beliefs after `observing' that it should follow the solution of the differential equation. In general  this is intractable and can only be handled approximately. By viewing  \cref{eqn:phi_gp_general_pde} as a loss function that measures the residual between $\diff{f}{t}$ and the operator $\NN_\theta \, f$ then the right hand side ($0$) are virtual observations. The \PDE can now be observed at a finite set of locations known as collocation points. This is a soft constraint (\ie  $\vf$ is not guaranteed to follow the differential equation), but it can handle non-linear and linear mechanisms. However, there are special cases, namely curl and divergence-free constraints, that can be solved exactly. This follows from properties of vectors fields, where $f$ defines a potential function where linear combinations of its partial derivatives define vector fields that enforce these properties. To handle both of these situations  we propose the following generative model
\begin{gather}
	\underbrace{\vphantom{\MF = \MW \, \left[\, \StateFProc_q \,\right]^\T} \TransF_n = \MW \cdot \left[\, \StateFProc_q(\MX_n)\,\right]^\T}_{\text{Linear Mixing}}, ~~  
	\underbrace{\vphantom{\MF = \MW \cdot \left[\vf_q\right]^\T} \StateFProc_q \sim \GP(\Zero, \KState_q) }_{\text{Independent \gp Priors}}, \\
	\underbrace{\vphantom{\MF = \MW \cdot \left[\vf_q\right]^\T} \vy^{(\mathcal{O})}_n= \MH_\mathcal{O} \, \TransF_n + \epsilon_{\mathcal{O}}}_{\text{Data}}, ~~
	\underbrace{\vphantom{\MF = \MW \cdot \left[\vf_q\right]^\T} \Zero^{(\mathcal{C})}_n = g(\TransF_n)}_{\text{Collocation Points}}, ~~
	\underbrace{\vphantom{\MF = \MW \cdot \left[\vf_q\right]^\T} \vy^{(\mathcal{B})}_n = \MH_\mathcal{B} \, \TransF_n + \epsilon_{\mathcal{B}}}_{\text{Boundary Values}},
\label{eqn:phi_gp_general_generative_model}
\end{gather}
where $\StateFProc_q$ are derivative \gps (see \cref{eqn:background_derivative_gp}) that are linearly mixed by $\MW \in \R^{(\NumTasks \, \Nd) \times (\NumLatents \, \Nd)}$, and $\MY^{(\mathcal{O})}, \Zero^{(\mathcal{C})} \in \R^{\N \times \NumTasks}$ are observations and collocation points over the $\NumTasks$ outputs and $\MY^{(\mathcal{B})} \in \R^{\N \times (\NumTasks \, \Nd)}$ are boundary values over the derivatives of each output. The observation matrices $\MH_\mathcal{O}, \MH_\mathcal{B}$ simply select the relevant parts of $\TransF_n$. For further details on notation see \cref{sec:app_vi_approx_deriv}. In many case we want to observe the solution of the differential equation exactly, however in some cases it may be required to add observation noise $\epsilon_{\mathcal{C}}$ to the collocation points, whether for numerical reasons or to model inexact mechanics. 
This is a flexible generative model where different assumptions and approximations will lead to various physics informed methods such as \AUTOIP, \EKS, \PMM, and \HELMHOLTZ that we will develop state space algorithms for. Additionally it is possible to learn missing physics by parameterising unknown terms in \cref{eqn:phi_gp_general_pde} through the \gp priors in \cref{eqn:phi_gp_general_generative_model} (see \cref{sec:app_unknown_physics}). 

\begin{example}[\EKS Prior and \PMM] \label{example:ode_eks_prior}
	We recover \EKS style generative models (see \citet{hennig_probabilistic:2022}) when the mixing weight is identity $\MW = \MI$, and $\epsilon_{\mathcal{C}}, \epsilon_{\mathcal{B}} \rightarrow 0$, and the non-linear transform $g$ is linearised. Let the prior be Markov $p(\StateF) = \prod^{\Nt}_{k} p(\StateF_k \mid \StateF_{k-1})$ with marginals $p(\StateF_k) = \nN{\StateF_k}{\vm^{-}_k}{\MP^{-}_k}$. By taking a first-order Taylor linearisation $g(\StateF_k) \simeq g(\vm^{-}_k) + \diff{g(\vm^{-}_k)}{\vm^{-}_k} \, \delta \StateF_k$ with $\delta \StateF_k \sim \nGauss{\Zero}{\MP^{-}_k}$ the joint is
	\begin{equation}
		p(\left[ \begin{matrix}
			\StateF_k \\
			\vg_k
		\end{matrix} \right]) \simeq \nN{
			\left[ \begin{matrix}
				\StateF_k \\
				\vg_k
			\end{matrix} \right]
		} {
			\left[ \begin{matrix}
				\vm^{-}_k \\
				g_k(\vm^{-}_k)
			\end{matrix} \right]
		} {
			\left[ \begin{matrix}
				\MI \\
				\diff{g(\vm^{-}_k)}{\vm^{-}_k}
			\end{matrix} \right] \, \MP^{-}_k \, \left[ \begin{matrix}
				\MI \\
				\diff{g(\vm^{-}_k)}{\vm^{-}_k}
			\end{matrix} \right]^\T
		}.
	\end{equation}
	This is now a form that can directly be implemented into an extended Kalman smoothing algorithm \citep{sarkka_filtering:2013}. When $\NumLatents>1$ the state $\StateF$ is constructed by stacking the individual states of each latent \citep{sarkka2019applied}. With linear \ODEs \EKS coincides with \PMM.  
\end{example}

\begin{example}[\HELMHOLTZ and Curl and Divergence-Free Vector Fields in $2$D]
	Let $\vv=[v_t, v_{s_1}, v_{s_2}]$ denote a $3$D-vector field, then curl indicates the tendency of a vector field to rotate and divergence at a specific point indicates the tendency of the field to spread out. Curl and divergence-free fields follow
	\begin{equation}
	\begin{aligned}
		\nabla \times \vv &= 0 ~~ (\text{curl free}),  ~~
		\nabla \cdot \vv &= 0 ~~ (\text{div. free})
	\end{aligned}
	\label{eqn:phi_gp_curl_div_free}
	\end{equation}
	where $\nabla = [\diff{}{\TimeIdx}, \diff{}{\SpaceIdx_1}, \diff{}{\SpaceIdx_2}]$.  Two basic properties of vector fields state that the divergence of a curl field and the curl of a derivative field are zero \citep{arfken_mathematical:2011}. Let $[f_1, f_2]$ be scalar potential functions then
	\begin{equation}
		\vv_{\text{curl}} = \nabla f_1 ~~ (\text{curl free}), ~~
		\vv_{\text{div}} = \nabla \times \nabla \, f_2  ~~ (\text{div. free})
\label{eqn:phi_gp_constrained_vector_field}
	\end{equation}
	define curl and divergence-free fields. In $2$D this simplifies to using the \emph{grad} and \emph{rot} operators over $\vv=[v_{s_1}, v_{s_2}]$ (see \citep{berlinghieri_helmholtz_gp:2023}). Placing \gp priors over $f_q$ we incorporate this into \cref{eqn:phi_gp_general_generative_model} by defining 
	\begin{equation}
		\MW_{\text{grad}} = \left[ \begin{matrix}
			1 & 0 \\
			0 & 1 
		\end{matrix} \right] \, \MH, ~~ \MW_{\text{rot}} = \left[ \begin{matrix}
			0 & 1 \\
			-1 & 0 \\
		\end{matrix} \right] \, \MH ~~ \text{where} ~~\MH ~~ \text{selects} ~~\left[\diff{f}{s_1}, \diff{f}{s_2} \right].
	\end{equation}
	\HELMHOLTZ is defined as the sum of \gp priors over $2$D curl and divergence-free fields \citep{berlinghieri_helmholtz_gp:2023}.
\end{example}

\subsection{A Spatio-Temporal State-Space Prior} \label{sec:st_diff_obs}

The generative model in \cref{eqn:phi_gp_general_generative_model} contains two complications: i) it includes potential non-lineararities, and ii) the independent priors are defined over latent functions with their partial derivatives which substantially increases the computational complexity. We wish to tackle both issues through state-space algorithms that are linear-in-time. We begin by deriving a state-space model that observes derivatives across space and time (see  \cref{sec:app_pigp_timeseries_setting} for the simpler time-series setting). In \cref{sec:pigp_cvi} we further derive a state-space variational lower bound that will enable computational speeds up in the spatial dimension. 

First, we show how Kronecker structure in the kernel allows us to rewrite the model as the solution to an \LTISDE. From the definition of $\DiffOp$, the separable covariance matrix has a repetitive structure that can be represented through a Kronecker product. The gram matrix is 
\begin{equation}
	\TDiff{\Diff{\MK(\vx, \vx)}} = \MK^{\DiffOp}_{\TimeIdx}(\vx_\TimeIdx, \vx_\TimeIdx) \, \kron \, \MK^{\DiffOp}_{\SpaceIdx}(\vx_\SpaceIdx, \vx_\SpaceIdx)
\end{equation}
where $\MK^{\DiffOp_\cdot}_\cdot = \left[\begin{matrix} \MK_{\cdot} &  \MK_{\cdot} \, {\DiffOpTilde_\cdot}^{*} \\ {\DiffOpTilde_\cdot} \, \MK_{\cdot} & {\DiffOpTilde_\cdot} \, \MK_{\cdot} \, {\DiffOpTilde_\cdot}^{*} \end{matrix} \right]$ and  $\DiffOpTilde_\cdot [\cdot] = (\DiffOp_\cdot [\cdot])_{1:}$ excludes the underlying latent function.  To find a Kronecker form of the gram matrix over $\MX$, we will exploit the fact that $\MX$ is on a spatio-temporal grid and that the kernel is separable. Due to the separable structure a derivative over either the spatio (or temporal) dimension only affects the corresponding kernel, and so when considering $\MX$, the gram matrix is still Kronecker structured:
\begin{equation}
	\diff{}{s} \, \MK(\vx, \vx) =  \MK_\Time(\vx, \vx) \cdot \diff{}{s} \, \MK_\Space(\vx, \vx) \Rightarrow
	\diff{}{s} \, \MK(\MX, \MX) =  \MK_\Time(\MX_\Time, \MX_\Time) \kron \diff{}{s} \, \MK_\Space(\MX_\Space, \MX_\Space).
\end{equation}
The full prior over (a permuted) $\MX$ is now given as
\begin{equation*}
	p(\StateFProc(\MX)) \, \permeq \, \nGauss{\Zero}{\MK_t^{\DiffOp}(\MX_\Time, \MX_\Time) \kron \MK_s^{\DiffOp}(\MX_\Space, \MX_\Space)}.
\end{equation*}
This is the form of a spatio-temporal Gaussian process with derivative kernels that can be immediately cast into a state-space form as in \cref{eqn:state_space_gp_form} where $\MH = \MI$, as we want to observe the whole state, not just $\vf$. 
The marginal likelihood and the \gp posterior can now be computed using standard Kalman filtering and smoothing algorithms with a computational time of $O(\Nt \cdot (\Ns \cdot \Nds \cdot d)^3)$. Inference in \PIGP now follows \cref{example:ode_eks_prior} by recognising that the filtering state consists of the spatial points with there spatio-temporal derivatives. The \EKS prior in \cref{example:ode_eks_prior} can now be simply extended to the \PDE setting by placing colocation points on a spatio-temporal grid \citep{kramer_pde_eks:2022}.

\subsection{A State-Space Variational Lower Bound (\vPIGP and \ekfPIGP)}   \label{sec:pigp_cvi}

We now derive a variational lower bound for \PIGP that maintains the computational benefits of state-space \gps. This acts as an alternative way of handling the non-linearity of $g$ in \cref{eqn:phi_gp_general_generative_model}, and will also enable the reduction of the cubic spatial computation complexity in \cref{sec:reducing_spatial_complexity}. We start by focusing on the single latent function setting ($\NumLatents=1$) and collect all terms that relate to observations in \cref{eqn:phi_gp_general_generative_model} with $p(\MY \mid \StateF) = \prod^N_n \, p(\vy^{(\mathcal{O})}_n | \MH_\mathcal{O} \, \TransF_n) \, p(\Zero^{(\mathcal{C})}_n | g(\TransF_n)) \, p( \vy^{(\mathcal{B})}_n | \MH_\mathcal{B} \, \TransF_n )
$. \VI frames inference as the minimisation of the Kullback–Leibler divergence  between the true posterior and an approximate posterior, which leads the optimisation of the  \ELBO \citep{hoffman_svi:2013}:
\begin{equation}
	\argmax_{q(\StateF \mid \param)} \, \LL = \nE{q(\StateF)}{
		\log \frac{
			p(\MY \mid \StateF) \, p(\StateF)
		} {
			q(\StateF)
		}
	}
\label{eqn:general_vi_elbo}
\end{equation}
where we define the approximate posterior $q(\vf \mid \param) \defeq \nN{\vf}{\vm}{\MS}$ as a free-form Gaussian with $\param = (\vm, \MS)$ and $\vm \in \R^{D\,N \times 1}$, $\MS \in \R^{D\,N \times D\,N}$. The aim is to represent the approximate posterior as a state-space \gp posterior, which will enable efficient computation of the whole evidence lower bound (\ELBO). We will achieve this through the use of natural gradients. The natural gradient preconditions the standard gradient with the inverse Fisher matrix, meaning the information geometry of the parameter space is taken into account, leading to faster convergence and superior performance \citep{amari_natgrad:1996,khan_cvi:2017,hensman_gp_for_big_data:2013}. For Gaussian approximate posteriors the natural gradient has a simple form \citep{hensman_fast_vi:2012}
\begin{equation}
\begin{aligned}
	\natp_{k} &= \lambda_{k-1} + \beta \, \diff{\LL}{\meanp_{k}} = 
		(1-\beta) \, \apxnatp_{k-1} + \beta \, \diff{\ELL}{\meanp_{k}} + \priornatp = \apxnatp + \priornatp
\end{aligned}
\label{eqn:nat_param_cvi_update}
\end{equation}
where $\natp = (\MS^{-1} \vm, \nicefrac{1}{2} \, \MS^{-1})$ and $\meanp = (\vm, \vm \, \vm^\T + \MS)$ are the natural and expectation parameterisations. This is known as conjugate variational inference (\CVI) as $\apxnatp$ represent the natural parameters for the conjugate prior $\priornatp$ \citep{khan_cvi:2017,chang_fast_vi:2020, hamelijnck_spatio_temporal:2021, wilkinson_bayes_newton:2021}. 
For now, we will assume that the likelihood is conjugate to ensure that $[\natp_k]_2$ is \psd, this will be relaxed in \cref{sec:ng_gauss_newton}. The derivative of the \ELL is 
\begin{equation}
	\diff{\text{ELL}}{[\meanp]_2} = 
		{\textstyle \sum^{\Nt, \Ns}_{\TimeIdx, \SpaceIdx}}
			\diff{}{[\meanp]_2} \, \nE{
				q
			}{
				\log p(
					\MY_{(\TimeIdx, \SpaceIdx)} \mid \StateF_{(\TimeIdx, \SpaceIdx)}
					)
			},
\end{equation}
where the expectation is under $q(\StateF_{(\TimeIdx, \SpaceIdx)})$, a $\Nd$ dimensional Gaussian over the spatio-temporal derivatives at location $\vx_{\TimeIdx, \SpaceIdx}$. Within the sum, the only elements of $[\meanp]_2$ whose gradient will propagate through the expectation are the $\Nd \times \Nd$ elements corresponding to these locations. These points are unique 
and so $\diff{\text{ELL}}{[\meanp]_2}$ has some (permutated) block-diagonal structure, hence \cref{eqn:nat_param_cvi_update} can be written as
\begin{equation}
\begin{aligned}
	q(\StateF) 
	&\propto {\textstyle \prod^{Nt}_t} \left[ \N(\apxy_t \mid \StateF_t, \apxv_t) \right] \, p(\StateF)
\end{aligned}
\label{eqn:pigp_cvi_dt_conjugate_form}
\end{equation}
where $\apxy_t$ is $\Nd$-dimensional. The natural gradient update, \ie $q(\StateF_t)$ in moment parameterisation, can now be computed using Kalman smoothing in $\BO(\Nt \cdot (\Ns \cdot \Nds \cdot d)^3)$. Collecting $\apxy = \nVec{[\apxy_t]},  \apxv=\nBlkdiag{[\apxv_t]}$, then the \ELBO can also be computed efficiently by substituting this form of $q(\StateF_t)$ in
\begin{equation} 	
\begin{aligned}
	\LL &= 
		\sum^{\Nt, \Ns}_{\TimeIdx, \SpaceIdx} \nE{q(\StateF_{(\TimeIdx, \SpaceIdx)})}{ \log p(
					\MY_{(\TimeIdx, \SpaceIdx)} \mid \StateF_{(\TimeIdx, \SpaceIdx)}
					)} - \sum^{Nt}_t\nE{q(\StateF_t)}{ \log \N(\apxy_t \mid \StateF_t, \apxv_t) } + \log p(\apxy \mid \apxv)
\end{aligned}
\end{equation}
where the first two terms only depend on $q(\Diff{\vf}_t)$ and the final term is simply a by-product of running the Kalman filter, leading to a dominant computational complexity of $\BO(N \cdot (\Ns \cdot \Nds \cdot d)^3)$. This cost is linear in the datapoints ($N$) because the expected log likelihood above decomposes across all spatio-temporal locations.  In summary we have shown that natural gradient is equivalent updating a block-diagonal likelihood that decomposes across time; hence the approximate posterior is computable via Kalman smoothing algorithms. Extending to multiple latent functions ($\NumLatents>1$) we define a full Gaussian approximate posterior that captures all correlations between the latent functions $q(\StateF_1, \cdots, \StateF_\NumLatents) \defeq \nN{\StateF_1, \cdots, \StateF_\NumLatents}{\vm}{\MS}$ where $\vm \in \R^{(N \times Q) \times 1}, \MS \in \R^{(N \times Q) \times (N \times Q)}$. All the observation models in \cref{eqn:phi_gp_general_generative_model} decompose across data points, hence \cref{eqn:pigp_cvi_dt_conjugate_form} is still block-diagonal and decomposes across time, except now each component is of dimension $\NumLatents \times \Nt$ as it encodes the correlations of spatial points and their spatio-temporal derivatives across the latent functions. We denote this model as \vPIGP and \ekfPIGP when using a \EKS prior (see \cref{example:ode_eks_prior}).

\begin{theorem}
	Let the approximate posterior be  (full) Gaussian $q(\StateF_1, \cdots, \StateF_\NumLatents) \defeq \nN{\StateF_1, \cdots, \StateF_\NumLatents}{\vm}{\MS}$ where $\vm \in \R^{(N \times Q) \times 1}, \MS \in \R^{(N \times Q) \times (N \times Q)}$. When $g$ is linear a single natural gradient step with $\beta=1$ recovers the optimal solution $p(\StateF_1, \cdots, \StateF_\NumLatents\mid \MY)$.
\label{thm:cvi_optimal}
\end{theorem}
We prove this in \cref{sec:app_cvi_nat_grad_optimal}. This result not only demonstrates the optimality of our proposed inference scheme in the linear Gaussian setting, but confirms that we recover batch models like \PMM and \HELMHOLTZ, as well as \EKS (see \cref{example:ode_eks_prior}). 

\section{Reducing the Spatial Computational Complexity} \label{sec:reducing_spatial_complexity}

We now propose three approaches that reduce the cubic computational complexity in the number of spatial derivatives and locations.  The first augments the process with inducing points that alleviate cubic costs associated with $\Ns$. The second is a structured variational approximation that defines the approximate posterior only over the temporal prior and alleviates cubic costs associated with $\Nds$. Finally, we introduce spatial mini-batching that alleviates linear $\Ns$ costs. When used in conjunction, the dominant computation cost is $\nBO{\Nt \cdot \Nds \cdot (\Ms \cdot \Ndt)^3}$. These approximations are not only useful for the state-space setting and can readily be applied to reduce the computational complexity for batch variational models (such as \AUTOIP). See \cref{sec:app_autoip_speedup} for more details.
\paragraph{Spatio-Temporal Inducing Points (\sPIGP)}

In this first approximation, denoted by \sPIGP, we augment the full prior $p(\StateF)$ with inducing points. By defining these inducing points on a spatio-temporal grid, we will show that we can still exploit Markov conjugate operations through natural gradients. Let $\StateU = \Diff{\vu} \in \R^{M \times \Nd}$ be inducing points at locations $\MZ \in \R^{M \times \D}$. From the standard \SVGP formulation \citep{hensman_gp_for_big_data:2013}, the \ELBO is 

\begin{equation}
	\LL = \nE{q(\StateF, \StateU)}{
		\log \frac{
			p(\MY \mid \StateF) \, p(\StateU)
		} {
			q(\StateU)
		}
	}
\end{equation}
where $q(\StateF, \StateU) \defeq p(\StateF \mid \StateU) \, q(\StateU)$. By defining the inducing points on a spatio-temporal grid at temporal locations $\MX_\TimeIdx \in \R^\Nt$ and spatial $\MZ_\SpaceIdx \in \R^{\Ms \times (\D-1)}$ then the marginal $p(\StateF \mid \StateU)$ is Gaussian with mean 
\begin{equation}
	\mu_{\MF \mid \MU} = \big[ \, \MI \kron \DiffKs{\MX}{\MZ}  \, (\DiffKs{\MZ}{\MZ})^{-1} \, \big] \, \StateU
\label{eqn:spigp_kronecker_conditional_mean}
\end{equation}
and variance given in \cref{eqn:app_sparse_kronecker_conditional_covariance}. This Kronecker structure allows us to again `decouple' space and time, leading to natural gradient updates with block size $\Ms \times \Nd$, reducing the computational complexity to $\BO(N \, (\Ms \cdot \Nds \cdot d)^3)$. For full details, see \cref{sec:app_st_cvi_sparsity_single_latent}.

\paragraph{Structured Variational Inference (\hsPIGP)}
This second approximation, denoted as \hsPIGP, defines the inducing points \emph{only} over the temporal derivatives. This is a useful approximation as it can drastically reduce the size of the filter state, making it more computationally and memory efficient. We begin by defining the joint prior as 
\begin{equation*}
	p(\MF, \DiffT{\vf}) = p(\MF \mid \DiffT{\vf}) \, p(\DiffT{\vf})
\end{equation*}
where $p(\MF \mid \DiffT{\vf})$ is a Gaussian conditional with mean
\begin{equation*}
	\E\left[\MF \mid \DiffT{\vf}\right] = \big[ \, 
		\MI \kron  \, \subDiffKs{\MX}{\MX} \, \MK_\Space(\MZ_s, \MZ_s)^{-1}	\, 
	\big]  \, \DiffT{\vf}, ~~ \text{with} ~~ \subDiffKs{\MX}{\MX} = \left[\begin{matrix}
 			\MK_\Space(\MX_s, \MZ_s) \\
 			\DiffS{\MK_\Space(\MX_s, \MZ_s)}
	\end{matrix} \right].
\end{equation*}
We then define a structured variational posterior
\begin{equation*}
	q(\StateF, \DiffT{\vf}) \defeq 
	p(\MF \mid \DiffT{\vf}) \, q(\DiffT{\vf}).
\end{equation*}
Substituting this into the \ELBO we see that all the terms with the prior spatial derivatives cancel
\begin{equation*}
	\nE{q}{
		\log \frac{
			p(\MY \mid \StateF) \, \cancel{p(\StateF \mid \DiffT{\vf})} \, p(\DiffT{\vf})
		} {
			\cancel{p(\MF \mid \DiffT{\vf})} \, q(\DiffT{\vf})
		}
	}
\end{equation*}
Again, the marginal $q(\DiffT{\vf})$ maintains Kronecker structure, enabling Markov conjugate operations, leading to a computational cost of $\BO(N \cdot \Nds \cdot (\Ns \cdot d)^3)$, see \cref{sec:app_st_cvi_structured_q_single_latent}. These variational approximations can simply be applied to non-state-space variational approximation, see \cref{sec:app_autoip_speedup}.
\paragraph{Spatial Mini-Batching}

A standard approach for handling big data is through mini-batching where the \ELL is approximated using only a data subsample \citep{hensman_gp_for_big_data:2013}. Directly appling mini-batching would be of little computation benefit because computation of the \ELBO requires running a Kalman smoother that iterates through all time points. Instead, we mini-batch by subsampling $B_{\SpaceIdx}$ spatial points
\begin{equation}
\begin{aligned}
	\ELL 
	 & \approx \sum^{\Nt}_{\TimeIdx} \frac{\Ns}{B_{\SpaceIdx}} \sum^{B_{\SpaceIdx}}_{i} \nE{q}{\log p(\MY_{\TimeIdx, \SpaceIdx} \mid \StateF_{\TimeIdx, i})}
\end{aligned}
\end{equation}
where $i$ is uniformly sampled. We used in conjunction with \sPIGP and \hsPIGP, 
this results in dominant costs of $\BO(\Nt \, (\Ms \cdot \Nds \cdot d)^3)$ and $\nBO{\Nt \cdot \Nds \cdot (\Ms \cdot d)^3}$ when $B_s \ll \Ns$.

\section{Handling the PSD Constraint} \label{sec:ng_gauss_newton}

As discussed in \cref{sec:pigp_cvi} when the differential equation is non-linear, the model is no longer conjugate and the resulting natural gradients are not guaranteed to result in \psd updates. This issue has received some attention in the literature \citep{salimbeni_natgrad:2018, tran_vb_manifold:2019, lin_psd_constraint:2020}, but these approaches do not maintain an efficient conjugate representation. 
One distinction is \cite{wilkinson_bayes_newton:2021}, which uses the Gauss-Newton approximation to maintain conjugate operations.
We now extend this to support spatial inducing points and non-linear transformations. Due to space we focus on \sPIGP, but see  \cref{sec:app_gauss_newton} for further  details. The troublesome term for the natural gradient update in \cref{eqn:nat_param_cvi_update} is the Jacobian of the \ELL \wrt to the second expectation parameter;
which is not guaranteed to be \psd unless the \ELL is log convex \citep{lin_psd_constraint:2020}. Focusing at a single location $n=(t,s)$:
\begin{equation*}
\begin{aligned}
	\diff{\ELL_n}{[\meanp_{k}]_2} = \diff{}{\MS_u}\nEtight{q(\StateU_t)}{
		\nE{p(\StateF_n | \StateU_t)}{
			\log p(\MY_n \mid \StateF_n)
		}
	} \\
\end{aligned}
\end{equation*}
we apply the Bonnet's and Price's theorem \citep{lin_stein_lemma:2019} to bring the differential inside the expectation and make a Gauss-Newton \citep{golub_gauss_newton:1973} approximation ensuring that the Jacobian is \psd
\begin{equation}
\begin{aligned}
	\diff{\ELL}{[\meanp_{t}]_2} &\approx \sum^N_{n,p} \nE{q(\StateU_t)}{
		\MJ_{n, p}^\T \, \MH_{n,p} \, \MJ_{n,p}
	}, ~~ \text{where} ~~ \MJ_{n, p} = \diff{g_{n}(\mu_n)}{\StateU_t}, ~~ \MH_{n, p} = \mdiffII{\log p(\MY_n \mid g_{n})}{g_{n}},
\end{aligned}
\label{eqn:cvi_gauss_newton_update}
\end{equation}
and $g_{n,p} = g(\StateU_n)$ (\cref{eqn:phi_gp_general_pde}) and $\mu_n$ is the mean of $p(\StateF_n \mid \StateU_t)$ (\cref{eqn:spigp_kronecker_conditional_mean}). When using spatial mini-batching \cref{eqn:cvi_gauss_newton_update} is also subsampled. 
\section{Related Work}

From the optimality of natural gradients, in the conjugate setting, we exactly recover batch \gp based models such as \citep{wahlstrom_magnetic_gp:2013, jidling_linearly_constrained_gps:2017, berlinghieri_helmholtz_gp:2023}.  Our inference scheme also applies to models that do not require derivative information \,  i.e. in $\Ndt = \Nds = 1$. As a special case, we recover \citep{hamelijnck_spatio_temporal:2021}, but we have extended the inference scheme to support spatial mini-batching, allowing big spatial datasets to be used. The linear weighting matrix can be used to define a  linear model of coregionalisation and its variants  \citep{bonilla_multi_task:2008,alvarez_vector:2012,moreno_hetrogenous_gps:2018, vanhatalo_additive_multi_task_gps:2020} and through appropriately designed functionals also non-linear variants \citep{wilson_gprn:2012}.

In \citet{alvarez_lfm:2009} \gp priors over the solution of differential equations are obtained through a stochastic forcing term but they only consider situations where the Greens function is available. In \citep{hartikainen_linear_lfm:2011,hartikainen_non_linear_lfm:2012, schmidt_state_space_joint:2021,kramer_state_space_ode:2021}, efficient state-space algorithms are derived but are limited to the temporal setting only. Similarly, \citet{heinonen_learning_odes:2018}, learn a `free-form ODE'. In the spatio-temporal setting \citet{kramer_pde_eks:2022} and \citet{duffin_low_rank_eks_fem:2022} (which builds \citep{girolami_statfem:2021}) derive extended Kalman filter algorithms. Additionally there are approaches to constraining \gps by linear differential equations \citep{lange2018algorithmic, albert2019gaussian, besginow2022constraining}. More generally than \citep{berlinghieri_helmholtz_gp:2023} in \citep{harkonen2023gaussian} \gp  priors over the solutions to linear \PDEs with constant coefficients are derived.

Beyond \gp based models, physics informed neural networks (PINNs) incorporate physics by constructing a loss function between the network and the differential equation at a finite set of collocation points \citep{raissi_hidden_pde_gp:2017}. This amounts to a highly complex optimisation problem \citep{krishnapriyan_characterizing:2021} bringing difficulties for training  \citep{wang2022and, wang_pinn_respect_causality:2024} and uncertainty quantification (UQ) \citep{edwards_pinns:2022}. Current approaches to quantifying uncertainty in PINNs are based on dropout \citep{zhang_pinn_dropout:2019} and conformal predictions \citep{podina_conformalized:2024}. In recent years UQ and deep learning has received much attention however is limited by its computational cost \citep{papamarkou_bayesian_position_paper:2024}.

\section{Experiments} \label{sec:exp}

We now examine the performance of our \PIGP methods on multiple synthetic and real-world datasets. We compare against a batch \gp (no physical knowledge) and current state-of-art methods \AUTOIP  and  \HELMHOLTZ.  We provide more details on all experiments in \cref{sec:app_experimental_details}.

\paragraph{Non-linear Damped Pendulum}

In this first synthetic example, we consider learning the non-linear dynamics of a damped swinging pendulum. 
This is described by a second-order differential equation
\begin{equation}
	\diffII{\theta}{t} +  \nSin{\theta} + b \, \dif{\theta}{t} = 0
\end{equation}
where $b > 0$. The first term is a non-linear forcing term, and the third is the damping term. With $b=0.2$ we simulate a solution using Euler's method \citep{butcher_numerical_methods:2016} and generate 20 points in $t \in [0, 6]$ for training and 200 in $t \in [6, 30]$ for testing, with additive Gaussian noise of variance $0.01$.

\begin{table}
	\begin{minipage}[t]{.49\textwidth} 
		\scriptsize

		\caption{Test performance on the simulated damped pendulum. Time is the total wall clock time in seconds.}

		\setlength{\tabcolsep}{3pt}
		\renewcommand{\arraystretch}{.75}
		
		\label{table:pendulum}
		\vskip 0.05in
		{

		\begin{sc}
		\begin{tabular}{cccccc}	
			\toprule
				 Model & Whiten & $\Nc$ & Time & RMSE &  NLPD \\
			\midrule
		\multirow{4}{*}{\PIGP} & \multirow{4}{*}{$-$} &$10$ & $96.33$ & $0.22$ & $-0.09$ \\
		& & $100$ & $112.2$ & $0.05$ & $-0.38$ \\
		& & $500$ & $138.98$ & $0.05$ & $-0.72$ \\
		& & $1000$ & $144.29$ & $0.06$ & $-0.79$ \\
		\cmidrule{2-6}
		\multirow{8}{*}{\AUTOIP} & \multirow{4}{*}{$\times$} & $10$ & $153.52$ & $0.35$ & $0.41$ \\
		& & $100$ & $195.25$ & $0.2$ & $-0.08$ \\
		& & $500$ & $1011.88$ & $0.33$ & $0.32$ \\
		& & $1000$ &$ 5134.13$ & $0.36$ & $0.41$ \\
		\cmidrule{2-6}
		 & \multirow{4}{*}{\checkmark} &$10$ & $164.58$ & $0.16$ & $-0.30$ \\
		& & $100$ & $208.81$ & $0.05$ & $-0.41$ \\
		& & $500$ & $1088.31$ & $0.05$ & $-0.75$ \\
		& & $1000$ & $5656.62$ & $0.05$ & $-1.39$ \\
				\bottomrule
			\end{tabular}
			\end{sc}
			}

	\end{minipage}
	\hfill
	\begin{minipage}[t]{.49\textwidth}
		\scriptsize

		\caption{Test performance on the magnetic field strength experiment. Results are computed \wrt to the first output. Time is the average epoch time in seconds.}
		\label{table:curl_free}
		\vskip 0.25in
		{
		\setlength{\tabcolsep}{3pt}
		\begin{sc}
		\begin{tabular}{c|ccc|ccc}	
			\toprule
				  & \multicolumn{3}{c}{Time} & \multicolumn{3}{c}{R Squared}  \\
				    Spatial Size   & $5$ & $10$ & $20$ & $5$ & $10$ & $20$\\
			\midrule
				\HELMHOLTZ & $0.21$ & $0.46$ & $2.37$ & $0.23$ & $0.97$ & $0.97$ \\
				\PIGP & $0.43$ & $0.60$ & $1.16$ & $0.21$ & $0.97$ & $0.97$ \\
				\sPIGP & $0.44$ & $0.44$ & $0.31$ & $0.23$ & $0.96$ & $0.97$ \\
				\hPIGP & $0.29$ & $0.40$ & $0.35$ & $0.65$ & $0.86$ & $0.93$ \\
				\hsPIGP & $0.29$ & $0.28$ & $0.16$ & $0.65$ & $0.61$ & $0.74$ \\
			\bottomrule
			\end{tabular}
			\end{sc}
			}
	\end{minipage}
\end{table}

We are interested in \emph{i)} the effect of the number of collocation points, \emph{ii)} the effect of the optimisation algorithm. To answer these questions, we compare against \AUTOIP on $[10, 100, 500, 1000]$ collocation points, with and without whitening. Results are tabulated in \cref{table:pendulum}. As expected, the predictive RMSE of all models decreases as the number of collocation points increases. Due to the cubic complexity of \AUTOIP, the total time significantly increases as the number of collocation points increases. For example, when using $1000$ collocation points, \AUTOIP is $\approx 39$ times slower than \PIGP. Interestingly, the un-whitened case performs poorly, possibly due to the nonlinearity of the differential equation making optimisation difficult. This indicates that either whitening or natural gradients are required to handle the non-linearity arising due to the differential equation.

\paragraph{Curl-free Magnetic Field Strength} %

\begin{figure*}[!t]
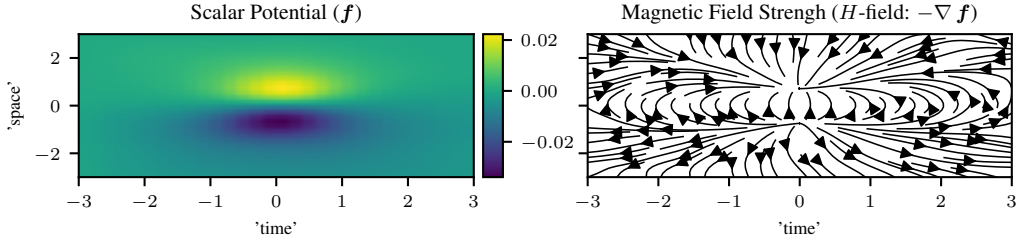

	\resizeinpgf{figs/curl_free}{curl_free_spatial.pgf}
	\caption{Curl free synthetic example. The left panel displays the learnt scalar potential functions by \PIGP with $\Ns=20$, and the right panel illustrates the associated vector field. }
\label{fig:curl_free_predictions}
\end{figure*}

In this experiment, we consider modelling the magnetic field strength of a dipole $H(r) = - \nabla \psi(r)$, where $\psi(r) = \nicefrac{\vm \cdot \vr}{|\vr|^3}$ is a scalar potential function \citep{chow_electromagnetic_theory:2006}.
Labelling the input dimensions as `time', `space' and `$z$', we let $\vm = [0, 1, 0]$ and generate observations from a spatio-temporal grid with $\Nt=50$, and $\Ns = [5, 10, 20]$, at $z=1$. $H(\vr)$ is a curl-free field and so we compare the curl free part of \HELMHOLTZ against \PIGP and its variants. \HELMHOLTZ and \PIGP are equivalent models (as this is the conjugate setting, \cref{thm:cvi_optimal}), and recover the same posterior and predictive distribution (up to numerical precision). However, due to the cubic-in-time complexity \HELMHOLTZ, at larger spatial sizes, is over $2$-times slower. The hierarchical approximation is substantially faster than \PIGP and performs similarly. As expected when introducing sparsity both \sPIGP and \hsPIGP are even faster; however, this is compensated by a slight drop in predictive performance. See \cref{fig:curl_free_predictions} and \cref{table:curl_free}. 

\begin{table}
	\setlength{\columnsep}{8pt}
	\setlength{\intextsep}{6pt}
	\centering 
	\scriptsize
	
	\caption{Test performance on the diffusion-reaction system. Time is the total wall clock time in seconds. \ekfPIGP significantly outperforms all models, and due to the \EKS prior only requires a $1$ epoch for inference. \hsPIGP  achieves the same performance as \AUTOIP but is over twice as fast.}
	\label{table:ac_res}
	{
	\def\arraystretch{1.0} %
	\setlength{\tabcolsep}{0.65em} %
	\begin{sc}
	\begin{tabular}{cccccc}	
		\toprule
			 Model &  RMSE & NLPD & CRPS & Time & Epochs\\
		\midrule
			\ekfPIGP & $\bf{0.09}$ & $\bf{-1.26}$ & $\bf{0.038}$ & $\bf{\sn{1.1}{3}}$ & $1$ \\ 
			\sPIGP & $0.19$ & $6.80$ & $0.093$ &$\sn{1.4}{4}$ & $20000$\\
			\hsPIGP & $0.17$ & $1.69$ & $0.077$ &$\sn{4.8}{3}$ & $20000$\\
			\AUTOIP & $0.17$ & $-0.29$ & $0.065$ & $\sn{1.1}{4}$ & $20000$\\
		\bottomrule
	\end{tabular}
	\end{sc}
	}
\end{table}

\paragraph{Diffusion-Reaction System} %

\begin{figure*}[!h]
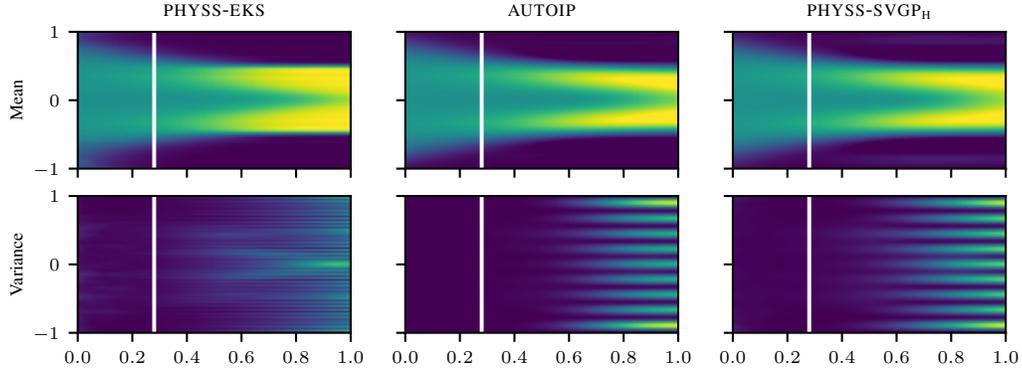

	\resizeinpgf{figs/ac}{ac.pgf}
	
	\caption{Results on the diffusion reaction system. The top row denotes the predictive mean, and the bottom the $95\%$ confidence intervals. The white line denotes where the training data ends. Only \ekfPIGP captures the sharp boundaries, due to the \IWP kernel. \hsPIGP recovers a similar solution to \AUTOIP but at half the computational cost.}

	\label{fig:exp_ac}
\end{figure*}

Consider a diffusion-reaction system given by an Allen-Cahn equation
\begin{equation}
	\diff{u
}{t} - 0.00001 \, \diffII{u}{x} + 5\, u^3 - 5\,u = 0
\end{equation}
where $x \in [-1, 1]$, $t\in [0, 1]$, $u(0, x) = x^2 \, \cos(\pi \, x)$, $u(t, -1) = u(t,1)$ and $\diff{u}{x}(t, -1)=\diff{u}{x}(t, 1)$. Following \citep{long_autoip:2022}, we use the solution provided by \citep{raissi_pinn:2019}	and sample $256$ training examples from $t \in [0, 0.28]$.
We compare \ekfPIGP (where $g$ is  linerized in the \EKS prior), \sPIGP and \hsPIGP against \AUTOIP. Following \citep{long_autoip:2022}, we use a learning rate of $0.001$ for Adam. For \AUTOIP, we place 100 collocation points across the whole input domain on a regular grid. For both \sPIGP, and \hsPIGP we require more collocation points in the temporal dimension and place them on a regular grid of size $20\times10$. For \ekfPIGP we use an integrated Wiener kernel (\IWP) on time \citep{sarkka2019applied} and place $100\times40$ collocation points. We are unable to place more collocation for \AUTOIP due to computational limits. Results are presented in \cref{fig:exp_ac} and \cref{table:ac_res}. \ekfPIGP requires only a single epoch and can better handle the sharp boundaries. Our method  \hsPIGP is over twice as fast as \AUTOIP whilst achieving similar predictive RMSE.

\paragraph{Ocean Currents} %

\begin{figure}[!h]
	\input{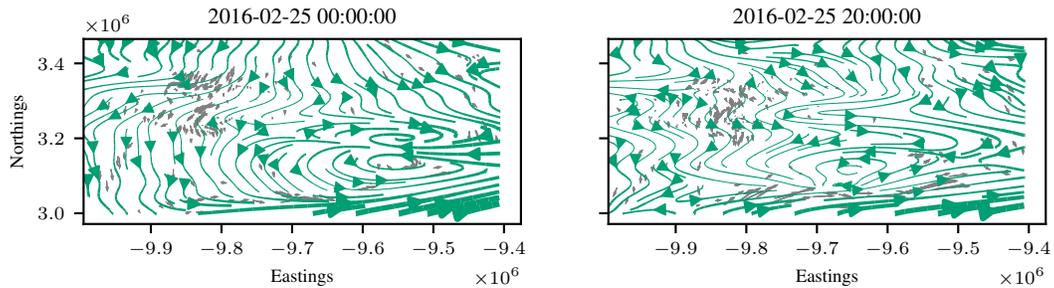}
	\caption{Predicted ocean currents by \hsPIGP. True observations are in grey, and predictions in green. The thickness of the line represents uncertainty and is computed by the L2 norm of the standard deviations across both outputs. %
	}	
	\label{eqn:laser_stream_plot}
\end{figure}

We now model oceanic currents in the Gulf of Mexico. We follow \citep{berlinghieri_helmholtz_gp:2023} and use the dataset provided by \citet{dasaro_ocean:2018} that has information from over $1,000$ buoys. We focus on the region in long. $[-90, -84.5]$, lat. $[26, 30]$ on 2016-02-25, by computing hourly averages. This results in $N=42,243$ observations, and we construct a test-train split on $0.1$ per cent of the data. 
It is infeasible to run  \HELMHOLTZ due to data size (in \citet{berlinghieri_helmholtz_gp:2023}, observations from only 19 buoys are used with $N=55$). However, we run \hsPIGP with $50$ spatial inducing points and a spatial mini-batch size of $10$, and plot results in \cref{eqn:laser_stream_plot}. Our predictions are in excellent agreement with the test data, achieving an \RMSE of $0.14$, \NLPD of $-0.52$, \CRPS of $0.078$, and an average run-time of $1.86(s)$ per epoch.

\section{Conclusion}

We introduced a physics-informed state-space \gp that integrates observational data with physical knowledge. Within the variational inference framework, we derived a computationally efficient algorithm that uses Kalman smoothing to achieve linear-in-time costs. To gain further computational speed-ups, we proposed three approximations with inducing points, spatial mini batching and structured variational posteriors. When used in conjunction, they allow us to handle large-scale spatiotemporal problems. The bottleneck is always the state size, where nearest neighbours \gps \citep{datta_nearest_neighbor_gps:2016,wu_nearest_neighbour:2022} could be explored. For highly non-linear problems, future directions could explore deep approaches \citep{salimbeni_dgp:2017} or more flexible kernel families \citep{wang_non_separable:2020}. One limitation is the use of the collocation method which is only enforcing the differential equation point wise, whilst  future work could look at the more general methods of weighted residuals \citep{pfortner_physics_informed:2023}.

\begin{ack}

OH acknowledges funding from The Alan Turing Institute PhD fellowship programme and the UKRI Turing AI Fellowship (EP/V02678X/1).
AS acknowledges support from the Research Council of Finland (339730).
TD acknowledges support from UKRI Turing AI Acceleration Fellowship (EP/V02678X/1) and a Turing Impact Award from the Alan Turing Institute.
The authors acknowledges the University of Warwick Research Technology Platform (aquifer) for assistance in the research described in this paper.
For the purpose of open access, the authors have applied a Creative Commons Attribution (CC-BY) license to any Author Accepted Manuscript version arising from this submission.

\end{ack}

\bibliographystyle{abbrvnat}

\begin{thebibliography}{77}
\providecommand{\natexlab}[1]{#1}
\providecommand{\url}[1]{\texttt{#1}}
\expandafter\ifx\csname urlstyle\endcsname\relax
  \providecommand{\doi}[1]{doi: #1}\else
  \providecommand{\doi}{doi: \begingroup \urlstyle{rm}\Url}\fi

\bibitem[Albert(2019)]{albert2019gaussian}
C.~G. Albert.
\newblock Gaussian processes for data fulfilling linear differential equations.
\newblock In \emph{Proceedings}, volume~33, page~5. MDPI, 2019.

\bibitem[Amari(1998)]{amari_natgrad:1996}
S.-i. Amari.
\newblock Natural gradient works efficiently in learning.
\newblock \emph{Neural Computation}, 10\penalty0 (2):\penalty0 251--276, 1998.

\bibitem[Arfken et~al.(2011)Arfken, Weber, and Harris]{arfken_mathematical:2011}
G.~B. Arfken, H.~J. Weber, and F.~E. Harris.
\newblock \emph{Mathematical Methods for Physicists: {A} Comprehensive Guide}.
\newblock Academic Press, 2011.

\bibitem[Berlinghieri et~al.(2023)Berlinghieri, Trippe, Burt, Giordano, Srinivasan, {\"O}zg{\"o}kmen, Xia, and Broderick]{berlinghieri_helmholtz_gp:2023}
R.~Berlinghieri, B.~L. Trippe, D.~R. Burt, R.~Giordano, K.~Srinivasan, T.~{\"O}zg{\"o}kmen, J.~Xia, and T.~Broderick.
\newblock Gaussian processes at the {H}elm (holtz): A more fluid model for ocean currents.
\newblock \emph{arXiv preprint arXiv:2302.10364}, 2023.

\bibitem[Besginow and Lange-Hegermann(2022)]{besginow2022constraining}
A.~Besginow and M.~Lange-Hegermann.
\newblock Constraining {G}aussian processes to systems of linear ordinary differential equations.
\newblock In \emph{Advances in Neural Information Processing Systems 35 (NeurIPS)}, pages 29386--29399. Curran Associates, Inc., 2022.

\bibitem[Black and Scholes(1973)]{black_options:1973}
F.~Black and M.~Scholes.
\newblock The pricing of options and corporate liabilities.
\newblock \emph{Journal of Political Economy}, 81\penalty0 (3):\penalty0 637--654, 1973.

\bibitem[Bonilla et~al.(2008)Bonilla, Chai, and Williams]{bonilla_multi_task:2008}
E.~V. Bonilla, K.~Chai, and C.~Williams.
\newblock Multi-task {G}aussian process prediction.
\newblock In \emph{Advances in Neural Information Processing Systems 20 (NIPS)}. Curran Associates, Inc., 2008.

\bibitem[Borthwick(2017)]{borthwick2017introduction}
D.~Borthwick.
\newblock \emph{Introduction to Partial Differential Equations}.
\newblock Universitext. Springer International Publishing, 2017.

\bibitem[Butcher(2016)]{butcher_numerical_methods:2016}
J.~Butcher.
\newblock \emph{Numerical Methods for Ordinary Differential Equations}.
\newblock Wiley, 2016.

\bibitem[Chang et~al.(2020)Chang, Wilkinson, Khan, and Solin]{chang_fast_vi:2020}
P.~E. Chang, W.~J. Wilkinson, M.~E. Khan, and A.~Solin.
\newblock Fast variational learning in state-space {G}aussian process models.
\newblock In \emph{30th {IEEE} International Workshop on Machine Learning for Signal Processing (MLSP)}, pages 1--6. {IEEE}, 2020.

\bibitem[Chow(2006)]{chow_electromagnetic_theory:2006}
T.~Chow.
\newblock \emph{Introduction to Electromagnetic Theory: A Modern Perspective}.
\newblock Jones and Bartlett Publishers, 2006.

\bibitem[Cockayne et~al.(2017)Cockayne, Oates, Sullivan, and Girolami]{cockayne_meshless_pde:2017}
J.~Cockayne, C.~Oates, T.~Sullivan, and M.~Girolami.
\newblock {Probabilistic numerical methods for PDE-constrained Bayesian inverse problems}.
\newblock \emph{AIP Conference Proceedings}, 1853\penalty0 (1), 06 2017.

\bibitem[Datta et~al.(2016)Datta, Banerjee, Finley, and Gelfand]{datta_nearest_neighbor_gps:2016}
A.~Datta, S.~Banerjee, A.~O. Finley, and A.~E. Gelfand.
\newblock Hierarchical nearest-neighbor {G}aussian process models for large geostatistical datasets.
\newblock \emph{Journal of the American Statistical Association}, 111\penalty0 (514):\penalty0 800--812, 2016.

\bibitem[Duffin et~al.(2022)Duffin, Cripps, Stemler, and Girolami]{duffin_low_rank_eks_fem:2022}
C.~Duffin, E.~Cripps, T.~Stemler, and M.~Girolami.
\newblock Low-rank statistical finite elements for scalable model-data synthesis.
\newblock \emph{Journal of Computational Physics}, 463:\penalty0 111261, 2022.

\bibitem[D’Asaro et~al.(2017)D’Asaro, Guigand, Haza, Huntley, Novelli, \"{O}zg\"{o}kmen, and Ryan]{dasaro_ocean:2018}
E.~D’Asaro, C.~Guigand, A.~Haza, H.~Huntley, G.~Novelli, T.~\"{O}zg\"{o}kmen, and E.~Ryan.
\newblock Lagrangian submesoscale experiment ({LASER}) surface drifters, interpolated to 15-minute intervals, 2017.
\newblock URL \url{https://data.gulfresearchinitiative.org/data/R4.x265.237:0001}.

\bibitem[Edwards(2022)]{edwards_pinns:2022}
C.~Edwards.
\newblock Neural networks learn to speed up simulations.
\newblock \emph{Communications of the ACM}, 65:\penalty0 27--29, 04 2022.

\bibitem[Eriksson et~al.(2018)Eriksson, Dong, Lee, Bindel, and Wilson]{eriksson_scaling_diff_gps:2018}
D.~Eriksson, K.~Dong, E.~Lee, D.~Bindel, and A.~G. Wilson.
\newblock Scaling {G}aussian process regression with derivatives.
\newblock In \emph{Advances in Neural Information Processing Systems 31 (NeurIPS)}. Curran Associates, Inc., 2018.

\bibitem[Girolami et~al.(2021)Girolami, Febrianto, Yin, and Cirak]{girolami_statfem:2021}
M.~Girolami, E.~Febrianto, G.~Yin, and F.~Cirak.
\newblock The statistical finite element method ({statFEM}) for coherent synthesis of observation data and model predictions.
\newblock \emph{Computer Methods in Applied Mechanics and Engineering}, 375:\penalty0 113533, 2021.

\bibitem[Golub and Pereyra(1973)]{golub_gauss_newton:1973}
G.~H. Golub and V.~Pereyra.
\newblock The differentiation of pseudo-inverses and nonlinear least squares problems whose variables separate.
\newblock \emph{SIAM Journal on Numerical Analysis}, 10\penalty0 (2):\penalty0 413--432, 1973.

\bibitem[Hamelijnck et~al.(2021)Hamelijnck, Wilkinson, Loppi, Solin, and Damoulas]{hamelijnck_spatio_temporal:2021}
O.~Hamelijnck, W.~J. Wilkinson, N.~A. Loppi, A.~Solin, and T.~Damoulas.
\newblock Spatio-temporal variational {G}aussian processes.
\newblock In \emph{Advances in Neural Information Processing Systems (NeurIPS)}. Curran Associates, Inc., 2021.

\bibitem[Harkonen et~al.(2023)Harkonen, Lange-Hegermann, and Raita]{harkonen2023gaussian}
M.~Harkonen, M.~Lange-Hegermann, and B.~Raita.
\newblock Gaussian process priors for systems of linear partial differential equations with constant coefficients.
\newblock In \emph{International Conference on Machine Learning}, pages 12587--12615. PMLR, 2023.

\bibitem[Hartikainen and S\"{a}rkk\"{a}(2011)]{hartikainen_linear_lfm:2011}
J.~Hartikainen and S.~S\"{a}rkk\"{a}.
\newblock Sequential inference for latent force models.
\newblock In \emph{Proceedings of the Twenty-Seventh Conference on Uncertainty in Artificial Intelligence}, UAI'11, page 311–318, Arlington, Virginia, USA, 2011. AUAI Press.

\bibitem[Hartikainen et~al.(2012)Hartikainen, Sepp\"{a}nen, and S\"{a}rkk\"{a}]{hartikainen_non_linear_lfm:2012}
J.~Hartikainen, M.~Sepp\"{a}nen, and S.~S\"{a}rkk\"{a}.
\newblock State-space inference for non-linear latent force models with application to satellite orbit prediction.
\newblock In \emph{Proceedings of the 29th International Coference on International Conference on Machine Learning}, ICML'12, page 723–730. Omnipress, 2012.

\bibitem[Heinonen et~al.(2018)Heinonen, Çagatay Yildiz, Mannerstr{\"o}m, Intosalmi, and L{\"a}hdesm{\"a}ki]{heinonen_learning_odes:2018}
M.~Heinonen, Çagatay Yildiz, H.~Mannerstr{\"o}m, J.~Intosalmi, and H.~L{\"a}hdesm{\"a}ki.
\newblock Learning unknown {ODE} models with {G}aussian processes.
\newblock In \emph{International Conference on Machine Learning (ICML)}, 2018.

\bibitem[Hennig et~al.(2022)Hennig, Osborne, and Kersting]{hennig_probabilistic:2022}
P.~Hennig, M.~Osborne, and H.~Kersting.
\newblock \emph{Probabilistic Numerics: {C}omputation as Machine Learning}.
\newblock Cambridge University Press, 2022.

\bibitem[Hensman et~al.(2012)Hensman, Rattray, and Lawrence]{hensman_fast_vi:2012}
J.~Hensman, M.~Rattray, and N.~D. Lawrence.
\newblock Fast variational inference in the conjugate exponential family.
\newblock In \emph{Advances in Neural Information Processing Systems (NIPS)}, pages 2888--2896. Curran Associates Inc., 2012.

\bibitem[Hensman et~al.(2013)Hensman, Fusi, and Lawrence]{hensman_gp_for_big_data:2013}
J.~Hensman, N.~Fusi, and N.~D. Lawrence.
\newblock Gaussian processes for big data.
\newblock In \emph{Proceedings of the Twenty-Ninth Conference on Uncertainty in Artificial Intelligence (UAI)}, pages 282--290. AUAI Press, 2013.

\bibitem[Hoffman et~al.(2013)Hoffman, Blei, Wang, and Paisley]{hoffman_svi:2013}
M.~D. Hoffman, D.~M. Blei, C.~Wang, and J.~Paisley.
\newblock Stochastic variational inference.
\newblock \emph{Journal of Machine Learning Research}, 14, 2013.

\bibitem[Jidling et~al.(2017)Jidling, Wahlstr\"{o}m, Wills, and Sch\"{o}n]{jidling_linearly_constrained_gps:2017}
C.~Jidling, N.~Wahlstr\"{o}m, A.~Wills, and T.~B. Sch\"{o}n.
\newblock Linearly constrained {G}aussian processes.
\newblock In \emph{Advances in Neural Information Processing Systems 30 (NeurIPS)}. Curran Associates, Inc., 2017.

\bibitem[Karniadakis et~al.(2021)Karniadakis, Kevrekidis, Lu, Perdikaris, Wang, and Yang]{karniadakis2021physics}
G.~E. Karniadakis, I.~G. Kevrekidis, L.~Lu, P.~Perdikaris, S.~Wang, and L.~Yang.
\newblock Physics-informed machine learning.
\newblock \emph{Nature Reviews Physics}, 3\penalty0 (6):\penalty0 422--440, 2021.

\bibitem[Khan and Lin(2017)]{khan_cvi:2017}
M.~E. Khan and W.~Lin.
\newblock Conjugate-computation variational inference: Converting variational inference in non-conjugate models to inferences in conjugate models.
\newblock In \emph{Proceedings of the International Conference on Artificial Intelligence and Statistics (AISTATS)}, 2017.

\bibitem[Kingma and Ba(2014)]{kingma_adam:2014}
D.~P. Kingma and J.~Ba.
\newblock Adam: {A} method for stochastic optimization.
\newblock \emph{arXv preprint arXiv:1412.6980}, 2014.

\bibitem[Kr\"{a}mer and Hennig(2021)]{kramer_state_space_ode:2021}
N.~Kr\"{a}mer and P.~Hennig.
\newblock Linear-time probabilistic solution of boundary value problems.
\newblock In \emph{Advances in Neural Information Processing Systems 34 (NeurIPS)}, pages 11160--11171. Curran Associates, Inc., 2021.

\bibitem[Kr{{\"a}}mer and Hennig(2024)]{kramer_stable_implementation:2024}
N.~Kr{{\"a}}mer and P.~Hennig.
\newblock Stable implementation of probabilistic {ODE} solvers.
\newblock \emph{Journal of Machine Learning Research}, 25\penalty0 (111):\penalty0 1--29, 2024.

\bibitem[Kr\"amer et~al.(2022)Kr\"amer, Schmidt, and Hennig]{kramer_pde_eks:2022}
N.~Kr\"amer, J.~Schmidt, and P.~Hennig.
\newblock Probabilistic numerical method of lines for time-dependent partial differential equations.
\newblock In \emph{Proceedings of the 25th International Conference on Artificial Intelligence and Statistics}, volume 151 of \emph{Proceedings of Machine Learning Research}, pages 625--639. PMLR, 2022.

\bibitem[Krishnapriyan et~al.(2021)Krishnapriyan, Gholami, Zhe, Kirby, and Mahoney]{krishnapriyan_characterizing:2021}
A.~S. Krishnapriyan, A.~Gholami, S.~Zhe, R.~Kirby, and M.~W. Mahoney.
\newblock Characterizing possible failure modes in physics-informed neural networks.
\newblock \emph{Advances in Neural Information Processing Systems 34 (NeurIPS)}, 2021.

\bibitem[Lange-Hegermann(2018)]{lange2018algorithmic}
M.~Lange-Hegermann.
\newblock Algorithmic linearly constrained {G}aussian processes.
\newblock In \emph{Advances in Neural Information Processing Systems 31 (NeurIPS)}, pages 2137--2148. Curran Associates, Inc., 2018.

\bibitem[Lin et~al.(2019)Lin, Khan, and Schmidt]{lin_stein_lemma:2019}
W.~Lin, M.~E. Khan, and M.~Schmidt.
\newblock Stein's lemma for the reparameterization trick with exponential family mixtures.
\newblock \emph{arXiv preprint arXiv:1910.13398}, 2019.

\bibitem[Lin et~al.(2020)Lin, Schmidt, and Khan]{lin_psd_constraint:2020}
W.~Lin, M.~Schmidt, and M.~E. Khan.
\newblock Handling the positive-definite constraint in the {B}ayesian learning rule.
\newblock In \emph{Proceedings of the 37th International Conference on Machine Learning}, Proceedings of Machine Learning Research. PMLR, 2020.

\bibitem[Long et~al.(2022)Long, Wang, Krishnapriyan, Kirby, Zhe, and Mahoney]{long_autoip:2022}
D.~Long, Z.~Wang, A.~Krishnapriyan, R.~Kirby, S.~Zhe, and M.~Mahoney.
\newblock {A}uto{IP}: A united framework to integrate physics into {G}aussian processes.
\newblock In \emph{Proceedings of the 39th International Conference on Machine Learning}, volume 162 of \emph{Proceedings of Machine Learning Research}. PMLR, 17--23 Jul 2022.

\bibitem[Mchutchon and Rasmussen(2011)]{mchutchon_input_noise:2011}
A.~Mchutchon and C.~Rasmussen.
\newblock Gaussian process training with input noise.
\newblock In \emph{Advances in Neural Information Processing Systems 25 (NeurIPS)}. Curran Associates, Inc., 2011.

\bibitem[Moreno-Mu\~{n}oz et~al.(2018)Moreno-Mu\~{n}oz, Art\'{e}s, and \'{A}lvarez]{moreno_hetrogenous_gps:2018}
P.~Moreno-Mu\~{n}oz, A.~Art\'{e}s, and M.~\'{A}lvarez.
\newblock Heterogeneous multi-output {G}aussian process prediction.
\newblock In \emph{Advances in Neural Information Processing Systems 31 (NeurIPS)}. Curran Associates, Inc., 2018.

\bibitem[Padidar et~al.(2021)Padidar, Zhu, Huang, Gardner, and Bindel]{padidar_scaling_diff_info:2021}
M.~Padidar, X.~Zhu, L.~Huang, J.~R. Gardner, and D.~S. Bindel.
\newblock Scaling {G}aussian processes with derivative information using variational inference.
\newblock In \emph{Advances in Neural Information Processing Systems (NeurIPS)}. Curran Associates, Inc., 2021.

\bibitem[Pan et~al.(2022)Pan, Mason, and Matar]{pan_dce:2021}
I.~Pan, L.~R. Mason, and O.~K. Matar.
\newblock Data-centric engineering: integrating simulation, machine learning and statistics. challenges and opportunities.
\newblock \emph{Chemical Engineering Science}, 249:\penalty0 117271, 2022.

\bibitem[Papamarkou et~al.(2024)Papamarkou, Skoularidou, Palla, Aitchison, Arbel, Dunson, Filippone, Fortuin, Hennig, Hern\'{a}ndez-Lobato, Hubin, Immer, Karaletsos, Khan, Kristiadi, Li, Mandt, Nemeth, Osborne, Rudner, R\"{u}gamer, Teh, Welling, Wilson, and Zhang]{papamarkou_bayesian_position_paper:2024}
T.~Papamarkou, M.~Skoularidou, K.~Palla, L.~Aitchison, J.~Arbel, D.~Dunson, M.~Filippone, V.~Fortuin, P.~Hennig, J.~M. Hern\'{a}ndez-Lobato, A.~Hubin, A.~Immer, T.~Karaletsos, M.~E. Khan, A.~Kristiadi, Y.~Li, S.~Mandt, C.~Nemeth, M.~A. Osborne, T.~G.~J. Rudner, D.~R\"{u}gamer, Y.~W. Teh, M.~Welling, A.~G. Wilson, and R.~Zhang.
\newblock Position: {B}ayesian deep learning is needed in the age of large-scale {AI}.
\newblock In \emph{Proceedings of the 41st International Conference on Machine Learning}, volume 235 of \emph{Proceedings of Machine Learning Research}, pages 39556--39586. PMLR, 2024.

\bibitem[Pf{\"o}rtner et~al.(2022)Pf{\"o}rtner, Steinwart, Hennig, and Wenger]{pfortner_physics_informed:2023}
M.~Pf{\"o}rtner, I.~Steinwart, P.~Hennig, and J.~Wenger.
\newblock Physics-informed {G}aussian process regression generalizes linear {PDE} solvers.
\newblock \emph{arXiv preprint arXiv:2212.12474}, 2022.

\bibitem[Podina et~al.(2024)Podina, Rad, and Kohandel]{podina_conformalized:2024}
L.~Podina, M.~T. Rad, and M.~Kohandel.
\newblock Conformalized physics-informed neural networks.
\newblock In \emph{ICLR 2024 Workshop on AI4Differential Equations in Science}, 2024.
\newblock URL \url{https://openreview.net/forum?id=ZoFWS818qG}.

\bibitem[Raissi and Karniadakis(2017)]{raissi_hidden_pde_gp:2017}
M.~Raissi and G.~E. Karniadakis.
\newblock Hidden physics models: Machine learning of nonlinear partial differential equations.
\newblock \emph{Journal of Computational Physics}, 2017.

\bibitem[Raissi et~al.(2019)Raissi, Perdikaris, and Karniadakis]{raissi_pinn:2019}
M.~Raissi, P.~Perdikaris, and G.~Karniadakis.
\newblock Physics-informed neural networks: A deep learning framework for solving forward and inverse problems involving nonlinear partial differential equations.
\newblock \emph{Journal of Computational Physics}, 378:\penalty0 686--707, 2019.

\bibitem[Rasmussen and Williams(2006)]{rasmussen_gpml:2006}
C.~Rasmussen and C.~Williams.
\newblock \emph{Gaussian Processes for Machine Learning}.
\newblock MIT Press, Cambridge, MA, USA, 2006.

\bibitem[Riihimäki and Vehtari(2010)]{riihimaki_monotonic_gp:2010}
J.~Riihimäki and A.~Vehtari.
\newblock Gaussian processes with monotonicity information.
\newblock In \emph{Proceedings of the Thirteenth International Conference on Artificial Intelligence and Statistics (AISTATS)}, volume~9 of \emph{Proceedings of Machine Learning Research}. PMLR, 2010.

\bibitem[Salimbeni and Deisenroth(2017)]{salimbeni_dgp:2017}
H.~Salimbeni and M.~Deisenroth.
\newblock Doubly stochastic variational inference for deep {G}aussian processes.
\newblock In \emph{Advances in Neural Information Processing Systems 30 (NeurIPS)}, pages 4588--4599. Curran Associates, Inc., 2017.

\bibitem[Salimbeni et~al.(2018)Salimbeni, Eleftheriadis, and Hensman]{salimbeni_natgrad:2018}
H.~Salimbeni, S.~Eleftheriadis, and J.~Hensman.
\newblock Natural gradients in practice: Non-conjugate variational inference in gaussian process models.
\newblock In \emph{Proceedings of the International Conference on Artificial Intelligence and Statistics (AISTATS)}, 2018.

\bibitem[S\"{a}rkk\"{a}(2011)]{sarkka_spde:2011}
S.~S\"{a}rkk\"{a}.
\newblock Linear operators and stochastic partial differential equations in gaussian process regression.
\newblock In \emph{Proceedings of the 21st International Conference on Artificial Neural Networks - Volume Part II}, ICANN'11, page 151–158. Springer-Verlag, 2011.

\bibitem[S{\"a}rkk{\"a} and Garc\'{i}a-Fern\'{a}ndez(2021)]{sarkka_parallel_smoother:2021}
S.~S{\"a}rkk{\"a} and A.~F. Garc\'{i}a-Fern\'{a}ndez.
\newblock Temporal parallelization of bayesian smoothers.
\newblock \emph{IEEE Transactions on Automatic Control}, 2021.

\bibitem[S{\"a}rkk{\"a} and Solin(2019)]{sarkka2019applied}
S.~S{\"a}rkk{\"a} and A.~Solin.
\newblock \emph{Applied Stochastic Differential Equations}.
\newblock Cambridge University Press, 2019.

\bibitem[Schmidt et~al.(2021)Schmidt, Kr\"{a}mer, and Hennig]{schmidt_state_space_joint:2021}
J.~Schmidt, N.~Kr\"{a}mer, and P.~Hennig.
\newblock A probabilistic state space model for joint inference from differential equations and data.
\newblock In \emph{Advances in Neural Information Processing Systems 34 (NeurIPS)}, pages 12374--12385. Curran Associates, Inc., 2021.

\bibitem[Schober et~al.(2014)Schober, Duvenaud, and Hennig]{schober_ode_runge_kutta:2014}
M.~Schober, D.~K. Duvenaud, and P.~Hennig.
\newblock Probabilistic {ODE} solvers with {R}unge-{K}utta means.
\newblock In \emph{Advances in Neural Information Processing Systems 27 (NeurIPS)}. Curran Associates, Inc., 2014.

\bibitem[Solak et~al.(2002)Solak, Murray-smith, Leithead, Leith, and Rasmussen]{solak_diff_gp:2002}
E.~Solak, R.~Murray-smith, W.~Leithead, D.~Leith, and C.~Rasmussen.
\newblock Derivative observations in {G}aussian process models of dynamic systems.
\newblock In \emph{Advances in Neural Information Processing Systems 15 (NIPS)}. MIT Press, 2002.

\bibitem[Solin(2016)]{solin_thesis:2016}
A.~Solin.
\newblock \emph{{Stochastic Differential Equation Methods for Spatio-Temporal Gaussian Process Regression}}.
\newblock Doctoral thesis, Aalto University, 2016.

\bibitem[Solin et~al.(2018)Solin, Kok, Wahlstr\"om, Sch\"on, and S\"arkk\"a]{solin_ambient_magnet:2018}
A.~Solin, M.~Kok, N.~Wahlstr\"om, T.~B. Sch\"on, and S.~S\"arkk\"a.
\newblock Modeling and interpolation of the ambient magnetic field by {G}aussian processes.
\newblock \emph{Transactions on Robotics}, 34\penalty0 (4):\penalty0 1112–1127, 2018.

\bibitem[Stocker(2011)]{stocker_introduction_to_climate_modelling:2011}
T.~Stocker.
\newblock \emph{Introduction to Climate Modelling}.
\newblock Springer Science \& Business Media, 2011.

\bibitem[Särkkä(2013)]{sarkka_filtering:2013}
S.~Särkkä.
\newblock \emph{Bayesian Filtering and Smoothing}.
\newblock Institute of Mathematical Statistics Textbooks. Cambridge University Press, 2013.

\bibitem[Tran et~al.(2021)Tran, Nguyen, and Nguyen]{tran_vb_manifold:2019}
M.-N. Tran, D.~H. Nguyen, and D.~Nguyen.
\newblock Variational {B}ayes on manifolds.
\newblock \emph{Statistics and Computing}, 31:\penalty0 1--17, 2021.

\bibitem[Tronarp et~al.(2018)Tronarp, Kersting, S{\"a}rkk{\"a}, and Hennig]{tronarp_prob_ode_new_perspective:2018}
F.~Tronarp, H.~Kersting, S.~S{\"a}rkk{\"a}, and P.~Hennig.
\newblock Probabilistic solutions to ordinary differential equations as nonlinear {B}ayesian filtering: a new perspective.
\newblock \emph{Statistics and Computing}, 29:\penalty0 1297--1315, 2018.

\bibitem[Vanhatalo et~al.(2020)Vanhatalo, Hartmann, and Veneranta]{vanhatalo_additive_multi_task_gps:2020}
J.~Vanhatalo, M.~Hartmann, and L.~Veneranta.
\newblock {Additive Multivariate Gaussian Processes for Joint Species Distribution Modeling with Heterogeneous Data}.
\newblock \emph{Bayesian Analysis}, 15:\penalty0 415 -- 447, 2020.

\bibitem[Villacampa-Calvo et~al.(2021)Villacampa-Calvo, Zald\'{\i}var, Garrido-Merch\'{a}n, and Hern\'{a}ndez-Lobato]{villacampa_noisy_input_gps:2021}
C.~Villacampa-Calvo, B.~Zald\'{\i}var, E.~C. Garrido-Merch\'{a}n, and D.~Hern\'{a}ndez-Lobato.
\newblock Multi-class {G}aussian process classification with noisy inputs.
\newblock \emph{Journal of Machine Learning Research}, 22\penalty0 (1), 2021.

\bibitem[Wahlstr\"{o}m et~al.(2013)Wahlstr\"{o}m, Kok, Sch\"{o}n, and Gustafsson]{wahlstrom_magnetic_gp:2013}
N.~Wahlstr\"{o}m, M.~Kok, T.~B. Sch\"{o}n, and F.~Gustafsson.
\newblock Modeling magnetic fields using {G}aussian processes.
\newblock In \emph{IEEE International Conference on Acoustics, Speech and Signal Processing}, pages 3522--3526, 2013.

\bibitem[Wang et~al.(2020)Wang, Hamelijnck, Damoulas, and Steel]{wang_non_separable:2020}
K.~Wang, O.~Hamelijnck, T.~Damoulas, and M.~Steel.
\newblock Non-separable non-stationary random fields.
\newblock In \emph{Proceedings of the 37th International Conference on Machine Learning (ICML)}, volume 119 of \emph{Proceedings of Machine Learning Research}, pages 9887--9897. PMLR, 2020.

\bibitem[Wang et~al.(2022)Wang, Yu, and Perdikaris]{wang2022and}
S.~Wang, X.~Yu, and P.~Perdikaris.
\newblock When and why {PINN}s fail to train: {A} neural tangent kernel perspective.
\newblock \emph{Journal of Computational Physics}, 449:\penalty0 110768, 2022.

\bibitem[Wang et~al.(2024)Wang, Sankaran, and Perdikaris]{wang_pinn_respect_causality:2024}
S.~Wang, S.~Sankaran, and P.~Perdikaris.
\newblock Respecting causality for training physics-informed neural networks.
\newblock \emph{Computer Methods in Applied Mechanics and Engineering}, 421:\penalty0 116813, 2024.

\bibitem[Wilkinson et~al.(2023)Wilkinson, S{\"a}rkk{\"a}, and Solin]{wilkinson_bayes_newton:2021}
W.~J. Wilkinson, S.~S{\"a}rkk{\"a}, and A.~Solin.
\newblock Bayes-{N}ewton methods for approximate {B}ayesian inference with {PSD} guarantees.
\newblock \emph{Journal of Machine Learning Research}, 24\penalty0 (83):\penalty0 1--50, 2023.

\bibitem[Wilson et~al.(2012)Wilson, Knowles, and Ghahramani]{wilson_gprn:2012}
A.~G. Wilson, D.~A. Knowles, and Z.~Ghahramani.
\newblock Gaussian process regression networks.
\newblock In \emph{Proceedings of the 29th International Coference on International Conference on Machine Learning}, page 1139–1146, 2012.

\bibitem[Wu et~al.(2022)Wu, Pleiss, and Cunningham]{wu_nearest_neighbour:2022}
L.~Wu, G.~Pleiss, and J.~P. Cunningham.
\newblock Variational nearest neighbor {G}aussian process.
\newblock In \emph{Proceedings of the 39th International Conference on Machine Learning (ICML)}, volume 162 of \emph{Proceedings of Machine Learning Research}, pages 24114--24130. PMLR, 2022.

\bibitem[Zhang et~al.(2019)Zhang, Lu, Guo, and Karniadakis]{zhang_pinn_dropout:2019}
D.~Zhang, L.~Lu, L.~Guo, and G.~E. Karniadakis.
\newblock Quantifying total uncertainty in physics-informed neural networks for solving forward and inverse stochastic problems.
\newblock \emph{Journal of Computational Physics}, 397:\penalty0 108850, 2019.

\bibitem[Álvarez et~al.(2009)Álvarez, Luengo, and Lawrence]{alvarez_lfm:2009}
M.~Álvarez, D.~Luengo, and N.~D. Lawrence.
\newblock Latent force models.
\newblock In \emph{Proceedings of the Twelth International Conference on Artificial Intelligence and Statistics (AISTATS)}, volume~5 of \emph{Proceedings of Machine Learning Research}, pages 9--16. PMLR, 2009.

\bibitem[Álvarez et~al.(2012)Álvarez, Rosasco, and Lawrence]{alvarez_vector:2012}
M.~A. Álvarez, L.~Rosasco, and N.~D. Lawrence.
\newblock Kernels for vector-valued functions: A review.
\newblock \emph{Foundations and Trends in Machine Learning}, 4\penalty0 (3):\penalty0 195--266, 2012.

\end{thebibliography}
\appendix
\clearpage
\section*{Appendices}

\section{Variational Approximation Derivation} \label{sec:app_vi_approx_deriv}
\subsection{Overview of Notation}

\begin{table}[h]
	\caption{Table of Notation}
	\label{table:app_table_of_notation}
\footnotesize
\begin{center}
\begin{tabularx}{\textwidth}{cc p{.6\textwidth}}
		\toprule
			Symbol & Size & Description\\
		\midrule
			$\NumData$ & -- & Number of observations. \\
			$Q$ & -- & Number of latent functions. \\
			$P$ & -- & Number of latent outputs. \\
			$\D$ & -- & Number of input features. \\
			$\Ns$ & -- & Number of spatial points. \\
			$\Nt$ & -- & Number of temporal points. \\
			$\Nds$ & -- & Number of spatial derivatives. \\
			$\Ndt$ & -- & Number of temporal derivatives. \\
			$\Nd = \Nds \cdot \Ndt$ & -- & Total number of spatio-temporal derivatives. \\
			$d$ & -- & State dimension. \\
			$B_{\SpaceIdx}$ & -- & Spatial batch size. \\
			$\Ms$ & -- & Number of spatial inducing points. \\
			$\MX$ & $\NumData \times \D$ & Input data matrix. \\ 
			$\Xs$ & $\Ns \times \D$ & Spatial Locations of training data. \\ 
			$\Xt$ & $\Nt$ & Temporal locations of training data. \\ 
			$\vx, \MX_n, \, \MX_{t, s}$ & $\D$ & Single training input. \\ 
			$\vx_t$ & -- & Temporal axis of a single training input location $\vx$. \\ 
			$\vx_s$ & $\D-1$ &  Spatial axes of a single training input location $\vx$.. \\ 
			$\MY$ & $\NumData \times \P$ & Output data matrix. \\  
			$\MY_n, \, \MY_{t, s}$ & $\P$ & Single training output. \\ 
			$\StateF$ & $\Ns \times d$ & Filtering state. \\
			$\MW$ & $(\P \times \Nd) \times (\Q \times \Nd)$ & Mixing matrix between $\Q$ latent \gps \\
			$\StateFProc_q(\MX_n)$ & $\Nd$ & Random vector of the $\Nd$ derivatives at location $\MX_n$ \\
			$\MF_n$ & $(\P \times \Nd)$ & Output of linearly mixed \gps. \\
			$g: \R^{\P \cdot \Nd} \to \R$ & -- & Differential equation defined using $\Nd$ spatio-temporal derivatives and $P$ outputs/states. \\
			$\MZ_\SpaceIdx$ & $\R^{\Ms \times (\D-1)}$ & Spatial Inducing Points. \\
			$\KS(\Xs, \Xs)$ & $\Ns \times \Ns$ & Spatial Kernel. \\ 
			$\KT(\Xt, \Xt)$ & $\Nt \times \Nt$ & Temporal Kernel. \\ 
			$\MK(\X, \X)$ & $\N \times \N$ & Spatio-Temporal Kernel. \\ 
			$\bar{\MK} = \TDiff{\Diff{\MK(\MX, \MX)}}$ & $(\N \cdot \Nd) \times (\N \cdot \Nd)$ & Spatio-temporal kernel over all $\N$ locations and $\Nd$ derivatives. \\
			$\MK_t^{\DiffOp}(\MX_\Time, \MX_\Time)$ & $(\Nt \times \Ndt) \times (\Nt \times \Ndt) $ & Gram matrix over temporal derivatives. \\
			$\MK_s^{\DiffOp}(\MX_\Time, \MX_\Time)$ & $(\Ns \times \Nds) \times (\Ns \times \Nds) $ & Gram matrix over spatial derivatives. \\
		\bottomrule
	\end{tabularx}
	\end{center}
\end{table}

\subsection{Layout of Vectors and Matrices}

We use a numerator layout for derivatives. Let $Q$ denote the number of independent latent functions and $\Nd$ the number of derivatives computed, and let $f_{q, d}$ denote the latent \gp for $d$'th derivative of the $q$'th latent function.
 We will need to keep track of the permutation of our data \wrt to space, time, and latent functions. Inspired by `row-major' and `column-major' layouts, we will use the following terminology that describes the ordering of the data across latent functions and time and space:
\begin{itemize}
	\item \textbf{latent-data:} $\MF = \MF_{\ld} =  [\MF_1(\MX), \cdots, \MF_Q(\MX)]$ with $\MF_q(\MX) = [\vf_{q, 1}(\MX), \cdots, \vf_{q, \Nd}(\MX)]$ which is ordered by stacking each of the latent functions on top of each other.
	\item \textbf{data-latent:} $\MF_{\dl} = [\MF_1(\MX_n), \cdots, \MF_Q(\MX_n)]^N_n$ which is ordered by stacking the latent functions evaluated at each data point across all data points.
	\item \textbf{time-space:} $[\vf(\gridX_\TimeIdx)]^{\Nt}_\TimeIdx$ which is ordered by stacking each of the input points at each time point on top of each other. This is only applicable when there is a single latent function.
	\item \textbf{latent-time-space:} $[\MF_1(\gridX_1), \cdots, \MF_1(\gridX_{\Nt}), \cdots, \MF_Q(\gridX_1), \cdots, \MF_Q(\gridX_{\Nt})]$
	\item \textbf{time-latent-space:} $[\MF_1(\gridX_\TimeIdx), \cdots, \MF_Q(\gridX_\TimeIdx)]^{\Nt}_\TimeIdx$
\end{itemize}
The default order will be latent-data (and latent-time-space for spatio-temporal problems). Since all of these are just simple permutations of each other, there exists a permutation matrix that permutes between any two of the layouts above. We use the function $\pi$ to denote a function that performs this permutation such that:
\begin{equation}
\begin{aligned}
	\MF_{\dl} &= \pi_{\ld \rightarrow \dl}(\MF_{\ld}) \\
	\MF_{\ld} &= \pi_{\dl \rightarrow \ld}(\MF_{\dl})
\end{aligned}
\end{equation}

\subsection{Timeseries Setting - Single Latent Function} \label{sec:app_pigp_timeseries_setting}

Let $\MX \in \R^{N \times D}, \MY \in \R^{N \times P}$ be input-output observations across $P$ outputs, where $\NumData = \Nt$. For now, we only consider the case where $Q=1$. We assume that $f$ has a state-space representation, and we denote its state with its $\Nd$ time derivatives as $\MF(\MX) = [ \vf(\MX), \diff{\vf(\MX)}{\MX}, \pdiffII{\vf(\MX)}{\MX}, \cdots ]$ in latent-data format. The vector $\MF(\MX)$ is of dimension $(\NumData \times \Nd)$. We also use the notation $\MF_n = \MF(\MX_n)$, which is a $\Nd$-dimensional vector of the derivatives at location $\MX_n$. The joint model is
\begin{equation}
	\bigg[\, \prod^N_n p(\MY_n \mid \MF_n, \DE) \, \bigg] \, p(\MF).
\end{equation}
At this point, we place no particular restriction on the form of the likelihood, aside from decomposing across $\NumData$. The prior $p(\MF)$ is a multivariate \gp of dimension $\NumData \times \Nd$
\begin{equation}
	p(\MF) = \nN{\MF}{\Zero}{\TDiff{\Diff{\MK(\vx, \vx)}}}	
\end{equation}
Let $q(\MF)$ be a free-form multivariate Gaussian of the same dimension as $p(\MF)$ then the corresponding \ELBO is:
\begin{equation}
\begin{aligned}
	\LL &= 
		\nE{q(\MF)}{
			\log \frac{
				p(\MY \mid \MF, \DE) \, p(\MF)
			}{
				q(\MF)
			}
		}, \\
		&= 	\nE{q(\MF)}{\log p(\MY \mid \MF, \DE)} - \nKL{q(\MF)}{p(\MF)}, \\
		&= \underbrace{\sum^N_n \nE{q(\MF_n)}{
		\log p(\MY_n \mid \MF_n, \DE)}}_{\ELL} - \underbrace{\vphantom{\sum^N_n}\nKL{q(\MF)}{p(\MF)}}_{\KLD},
\end{aligned}
\end{equation}
and the marginal $q(\MF_n)$ is a $\Nd$-dimensional Gaussian corresponding to the $n$'th observation. The natural gradients are
\begin{alignat}{3}
	 \apxnatp & \leftarrow (1-\beta) \, \apxnatp + \beta \, \diff{\ELL}{\meanp}  &&\Bigl\} ~~ \text{Surrogate likelihood update} \label{eqn:app_timeseries_cvi_natural_gradient}\\
	\natp &  \leftarrow \apxnatp + \priornatp &&\Bigl\} ~~ \text{Surrogate model update} \label{eqn:app_timeseries_cvi_conjugate_model}
\end{alignat}
where $\apxnatp = \left[ [\apxnatp]_1, [\apxnatp]_2\right]^\T$ and $[\apxnatp]_1$ is an $(\NumData \times \Nd)$ vector and $[\apxnatp]_2$ an $(\NumData \times \Nd) \times (\NumData \times \Nd)$ matrix. \cref{eqn:app_timeseries_cvi_conjugate_model} is a sum of natural parameters, and so is the conjugate Bayesian update. Naively computing this would yield no computation speed up as the computation cost would be cubic $\BO(N^3)$. However, the natural parameters of the likelihood ($\apxnatp$) are guaranteed to be block diagonal, one block per data point (if $\apxnatp_0$ is initialised as so). This immediately implies that \cref{eqn:app_timeseries_cvi_conjugate_model} can be computed using efficient Kalman filter and smoothing algorithms. The structure of $\apxnatp$ depends on the gradient of the expected log-likelihood $\diff{\ELL}{\meanp}$. Expanding this out
\begin{equation}
\begin{aligned}
	\diff{\ELL}{[\meanp]_2} = \sum^N_n \diff{
		\nE{q(\MF_n)}{\log p(\MY_n \mid \MF_n, \DE)}
	}{[\meanp]_2} = \sum^N_n \apxmeanp_n
\end{aligned}
\end{equation}
where each component $\apxmeanp_n$ is a $(\NumData \times \Nd) \times (\NumData \times \Nd)$ matrix that only has $\Nd \times \Nd$ non-zero entries; as these are the only elements that directly affect $\MF_n$. Collecting all these submatrices into a block diagonal matrix, we have a matrix in data-latent format, however, $\diff{\ELL}{\meanp}$ is in latent-data, and so all we need to do is permute by $\MP$:
\begin{equation}
	\diff{\ELL}{\meanp} =   \pi_{\dl \rightarrow \ld}\left( \, \blkdiag[\, \apxmeanp_1, \cdots, \apxmeanp_N\, ] \, \right).
\end{equation}
Converting from natural to moment paramerisation the surrogate update is:
\begin{equation}
\begin{aligned}
	q(\MF) 
		&\propto \nN{\apxy}{\MF}{\apxv} \, p(\MF) \\
		&= \bigg[\prod^N_n \, \nN{\apxy_n}{\MF_n}{\apxv_n} \, \bigg] \, p(\MF)
\end{aligned}
\end{equation}
where $\apxy_n$ is a $\Nd$-dimensional vector, and $\apxv_n$ is a $\Nd \times \Nd$ matrix, and efficient Kalman filtering and smoothing algorithms can be used to compute the surrogate model update. Substituting $q(\MF)$ back into the \ELBO it further simplifies:
\begin{equation}
\begin{aligned}
	\LL &= 
		\nE{q(\MF)}{
			\log \frac{
				p(\MY \mid \MF, \DE) \, \cancel{p(\MF)} \, p(\apxy \mid \apxv)
			}{
				\nN{\apxy}{\MF}{\apxv} \, \cancel{p(\MF)}
			}
		} \\
		&= \sum^{\NumData}_{n} \nE{q(\MF_n)}{ \log p(\MY_n \mid \MF_n, \DE)} -  \sum^{\NumData}_{n} \nE{q(\MF_n)}{ \log \nN{\apxy_n}{\MF_n}{\apxv_n}} + \log p(\apxy \mid \apxv)
\end{aligned}
\end{equation}
each term can be computed efficiently as the by-product of the Kalman filtering and smoothing algorithm used to compute $q(\MF)$.

\subsection{Timeseries Setting - Multiple Latent Functions} \label{sec:app_pigp_timeseries_setting_multiple_latent}

We now generalise the previous section to handle multiple independent latent functions, \ie $\NumLatents > 0$. The model prior now has the form
\begin{equation}
	p(\MF) = \prod^\NumLatents_q p(\MF_q)
\end{equation}
where $p(\MF_q)$ is a prior over $\vf_{q, 1}$ and its $\Nd$ partial deriatives. We consider two approaches: a mean-field approximate posterior and a full Gaussian. 

The first approach defined mean-field approximate posterior $q(\MF) \defeq \prod^\NumLatents_q q(\MF_q)$ where each $q(\MF_q)$ is a free-form Gaussian of dimension $(\NumData \times \Nd)$. The natural gradient updates are now simply applied to each component $q(\MF_q)$ separately, and we essentially follow the update set out in \cref{sec:app_pigp_timeseries_setting}. 

The second approach is a full-Gaussian approximate posterior where $q(\MF)$ is a $(\NumLatents \times \Nd \times \NumData)$-dimensional free-form Gaussian. In this case the \ELL is
\begin{equation}
	\ELL = \sum^N_n \nE{q(\MF_n)}{\log p(\MY_n \mid \MF_n, \DE)}
\end{equation}
where $q(\MF_n)$ is of dimension $(\NumLatents \times \Nd)$. This implies that the gradient of the \ELL $\diff{\ELL}{[\meanp]_2} = \sum^N_n \apxmeanp_n$ where $\apxmeanp_n$ now has $(\NumLatents \times \Nd) \times (\NumLatents \times \Nd)$ non-zero entries. Switching to moment parameterisation
\begin{equation}
	q(\MF) \propto \bigg[\prod^N_n \, \nN{\apxy_n}{\MF_n}{\apxv_n} \, \bigg] \, p(\MF)
\end{equation}
where $\apxy_n$ is of dimension $(\NumLatents \times \Nd)$ and $\apxv_n$ is $(\NumLatents \times \Nd) \times (\NumLatents \times \Nd)$. We can still use state-space algorithms by simply stacking the states corresponding to each $\MF_q$ \citep{sarkka2019applied}.

\subsection{Spatio-temporal Data - Single Latent Function}

We now turn to the spatio-temporal setting where $\MX, \MY$ are spatio-temporal observations on a spatio-temporal grid ordered in time-space format. We now derive the conjugate variational algorithm for \sPIGP and \hsPIGP. The algorithms for \PIGP and \hPIGP are recovered as special cases when $\MZ=\Xs$.

\subsubsection{Spatial Derivative Inducing Points} \label{sec:app_st_cvi_sparsity_single_latent}

We follow the standard sparse variational GP procedure and augment that prior with inducing points $\MU = \Diff{\vu}$ at locations $\MZ$. We require that the inducing points are defined on a spatial-temporal grid at the same temporal points as the data $\MX$, such that $\MZ = [\MZ_t]^{\Nt}_t$. This is required to ensure Kronecker structure between the inducing points and the data. The joint model is 
\begin{equation}
	p(\MY \mid \MF) \, p(\MF \mid \MU) \, p(\MU)
\end{equation}
where 
\begin{equation}
\begin{aligned}
	p(\MU) &= \nN{\MU}{\Zero}{
		\MK^{\DiffOp}_\Time(\MX_\TimeIdx, \MX_\TimeIdx) \kron \MK^{\DiffOp}_\Space(\MZ_\SpaceIdx, \MZ_\SpaceIdx)
	}, \\
	p(\MF \mid \MU) &= \nN{\MF}{\mu_{\MF \mid \MU}}{\Sigma_{\MF \mid \MU}} 
\end{aligned}
\end{equation}
and the conditional mean and covariance are given by
\begin{equation}
\begin{aligned}
	\mu_{\MF \mid \MU} &= \big[ \, 
		\DiffKt{\MX}{\MX} \kron \DiffKs{\MX}{\MZ} 
	\, \big] \, \big[\, 
		\DiffKt{\MX}{\MX} \kron \DiffKs{\MZ}{\MZ}  
	\, \big]^{-1} \, \MU \\
	&= \big[ \, \MI \kron \DiffKs{\MX}{\MZ}  \, (\DiffKs{\MZ}{\MZ})^{-1} \, \big] \, \MU, \\
\end{aligned}
\label{eqn:app_sparse_kronecker_conditional_mean}
\end{equation}
and
\begin{equation}
\begin{aligned}
	\Sigma_{\MF \mid \MU} = 
		&\big[ \, \DiffKt{\MX}{\MX} \kron \DiffKs{\MX}{\MX} \, \big] \\
		&- 
		\big[ \, \MK^{\kron}_{\MX,\MZ} \, \big] \, 
		\big[\, \DiffKt{\MX}{\MX} \kron \DiffKs{\MZ}{\MZ}    \, \big]^{-1} \, 
		\big[ \, \MK^{\kron}_{\MX,\MZ} \, \big]^\T 
\end{aligned}
\end{equation}
where $\MK^{\kron}_{\MX,\MZ} = \big[ \, \DiffKt{\MX}{\MX} \kron \DiffKs{\MX}{\MZ}   \, \big]$ which simplifies to
\begin{equation}
\begin{aligned}
	\Sigma_{\MF \mid \MU} = \DiffKt{\MX}{\MX} \kron \big[ \,
		\DiffKs{\MX}{\MX} - \DiffKs{\MX}{\MZ} \, \DiffKs{\MZ}{\MZ}^{-1} \, \DiffKs{\MZ}{\MX}
	\, \big]
\end{aligned}
\label{eqn:app_sparse_kronecker_conditional_covariance}
\end{equation}
Due to the Kronecker structure, the marginal at time $t$ only depends on the inducing points in that time slice so we can still get a CVI-style update that can be computed using a state-space model. To see why we again look at the Jacobian of the \ELL: $\diff{\ELL}{[\meanp]_2} = \sum^N_n \apxmeanp_n$ where $\apxmeanp_n$ now has $(\Nd \times M) \times (\Nd \times M)$ non-zero entries, which corresponding to needed all $M$ spatial inducing points with there derivatives to predict at a single time point. This is similar to the time series setting, except we have now predicted in space to compute marginals of $q(\MF)$. To be complete, we write that the marginal $q(\MU)$ is
\begin{equation}
	q(\MU) \propto \bigg[\prod^{\Nt}_t \, \nN{\apxy_t}{\MF_t}{\apxv_t} \, \bigg] \, p(\MF)
\end{equation}
where $\apxy_t$ and $\MF_t$ are vectors of dimension $(\Nd \times M)$, and $\apxv_t$ is a matrix of dimension $(\Nd \times M) \times (\Nd \times M)$. The marginals $q(\MU_t)$, and the corresponding marginal likelihood $p(\apxy \mid \apxv)$ can be computed by running a Kalman filter and smoother in $\BO(\Nt \cdot (\Ms \cdot \Nds \cdot d)^3)$. The marginal $q(\MF) = \nN{\MF}{\mu_{\MF}}{\Sigma_{\MF}}$ where 
\begin{equation}
	\mu_{\MF} = \big[ \, \MI \kron \DiffKs{\MX}{\MZ}  \, (\DiffKs{\MZ}{\MZ})^{-1} \, \big] \, \vm
\end{equation}
and
\begin{equation}
	\Sigma_{\MF} = \Sigma_{\MF \mid \MU} + \left[ \MI \kron  \DiffKs{\MX}{\MZ}  \, (\DiffKs{\MZ}{\MZ})^{-1} \right] \, \MS \, \left[ \MI \kron  \DiffKs{\MX}{\MZ}  \, (\DiffKs{\MZ}{\MZ})^{-1} \right]^\T
\end{equation}

\subsubsection{Structured Approximate Posterior With Spatial Inducing Points}
\label{sec:app_st_cvi_structured_q_single_latent}

We now derive the algorithm for the case of the structured approximate posterior with spatial inducing points. The key is to define the free-form approximate posterior over the inducing points and their temporal derivatives and then use the model conditional to compute the spatial derivatives. The model is
\begin{equation}
	p(\MY \mid \MF) \, p(\MF \mid \DiffT{\vu}) \, p(\DiffT{\vu}).
\end{equation}
Each term is
\begin{equation}
\begin{aligned}
	p(\DiffT{\vu}) &= \nN{\DiffT{\vu}}{\Zero}{\TDiffT{\DiffT{\MK(\MZ, \MZ)}}}, \\
	p(\MF \mid \DiffT{\vu}) &= \nN{\MF}{\mu_{\MF \mid \MU_t}}{\Sigma_{\MF \mid \MU_t}} 
\end{aligned}
\end{equation}
where 
\begin{equation}
\begin{aligned}
	\mu_{\MF \mid \MU_t} &= \big[ \, 
		\DiffKt{\MX}{\MX} \kron  \, \subDiffKs{\MX}{\MX}	\, \big] \, \big[\, 
		\DiffKt{\MX}{\MX} \kron \MK_\Space(\MZ_s, \MZ_s) 
	\, \big]^{-1} \, \DiffT{\vu} \\
&= \big[ \, 
		\MI \kron  \, \subDiffKs{\MX}{\MX} \, \MK_\Space(\MZ_s, \MZ_s)^{-1}	\, 
	\big]  \, \DiffT{\vu}
\end{aligned}
\label{eqn:app_structured_cvi_kronecker_conditional_mean}
\end{equation}
and
\begin{equation}
\begin{aligned}
	\Sigma_{\MF \mid \MU_t} = 
		&\big[ \, \DiffKt{\MX}{\MX} \kron \DiffKs{\MX}{\MX} \, \big] \\
		&- 
		\big[ \, \widetilde{\MK}^{\kron}_{\MX, \MZ_s} \, \big] \, 
		\big[\, \DiffKt{\MX}{\MX} \kron \MK_\Space(\MZ_s, \MZ_s)  \, \big]^{-1} \, 
		\big[ \, \widetilde{\MK}^{\kron}_{\MX, \MZ_s}  \, \big]^\T 
\end{aligned}
\label{eqn:app_structured_cvi_kronecker_conditional_variance}
\end{equation}
where $\widetilde{\MK}^{\kron}_{\MX, \MZ_s} =\DiffKt{\MX}{\MX} \kron \, \subDiffKs{\MX}{\MZ_s}$ and
\begin{equation}
	\subDiffKs{\MX}{\MX} = \left[\begin{matrix}
 			\MK_\Space(\MX_s, \MZ_s) \\
 			\DiffS{\MK_\Space(\MX_s, \MZ_s)}
	\end{matrix} \right].
\end{equation}
The approximate posterior is defined as
\begin{equation}
	q(\MF, \DiffT{\vu}) = p(\MF \mid \DiffT{\vu}) \, q(\DiffT{\vu})
\end{equation}
where $q(\DiffT{\vu}$ is a free-form Gaussian of dimension $
(Nd \times \Nt \times \Ms)$. The rest of the derivation simply follows \cref{sec:app_st_cvi_sparsity_single_latent} by simpling substituting $\subDiffKs{\MX}{\MZ_s}$ into the corresponding conditionals. The final result is that the approximate posterior decomposes as
\begin{equation}
	q(\MU) \propto \bigg[\prod^{\Nt}_t \, \nN{\apxy_t}{\MF_t}{\apxv_t} \, \bigg] \, p(\MF)
\end{equation}
where $\apxy_t$ and $\MF_t$ are vectors of dimension $(\Ndt \times \Ms)$, and $\apxv_t$ is a matrix of dimension $(\Ndt \times \Ms) \times (\Ndt \times \Ms)$. The marginals $q(\MU_t)$, and the corresponding marginal likelihood $p(\apxy \mid \apxv)$ can be computed by running a Kalman filter and smoother in $\BO(\Nt \cdot (\Ms \cdot d)^3)$, which compared to \cref{sec:app_st_cvi_sparsity_single_latent} is \emph{not} cubic in the number of spatial derivatives.

\subsubsection{Gauss-Newton Natural Gradient Approximation} \label{sec:app_gauss_newton}

We now provide the full derivation of the Gauss-Newton approximation of the natural gradient used to ensure \psd updates. We will make use of the following identities, known as the Bonnet and Price theorems (see,  \citep{lin_stein_lemma:2019}),
\begin{align}
	\diff{}{\mu} \, \nE{q(\vf \mid \mu, \Sigma)}{\ell(\vf)} &= \nE{q(\vf \mid \mu, \Sigma)}{ \diff{}{\vf} \, \ell(\vf)} \\
	\diff{}{\Sigma} \, \nE{q(\vf \mid \mu, \Sigma)}{\ell(\vf)} &= \frac{1}{2} \, \nE{q(\vf \mid \mu, \Sigma)}{ \hess{}{\vf} \, \ell(\vf)} 
\end{align}
which describes how to bring derivatives inside expectations. To ease notations, we work with a more general description of the model presented in the main paper, where we have multiple independent latent functions and use $T_p$ to denote likelihood-specific functions which, for example, can be used to represent $\DE$ or as the identity of standard Gaussian likelihoods. The model is
\begin{equation}
\begin{aligned}
	p(\vu_q) &= \nN{\vu_q}{0}{\MK_q} \\
	p(\vf_q \mid \vu_q) &= \nN{\vf_q}{\mu_{\vf_q \mid \vu_q}}{\Sigma_{\vf_q \mid \vu_q}} \\
	\MY_{n,q} &= p(\MY_{n,q} \mid T_p(\vf_{n,1}, \hdots, \vf_{n,Q}))
\end{aligned}
\end{equation}
where the shapes are $\vu_q \in \R^{M}$, $\vf_q \in \R^{N}$, $T_p:\R^Q \rightarrow \R^P$, $\MY \in \R^{N \times P}$, and $\MY_{n, p} \in \R$. The variational approximation is 
\begin{equation}
	q(\MU) = \nN{\MU}{\vm}{\MS} 
\end{equation}
where $\MU = [\vu_1, \cdots, \vu_Q]$, $\vm \in \R^{QM \times 1}$ and $\MS \in \R^{QM \times QM}$. Let $\MF = [\vf_1, \hdots, \vf_Q]$. The expected log-likelihood of the variational approximation is
\begin{equation}
\begin{aligned}
	\ELL &= \nE{q(\MU)}{
		\nE{p(\MF \mid \MU)}{
			\sum_{n, p} \log p(\MY_{n, p} \mid T_p(\MF_{n, p}))
		}
	} \\
	&= \sum_{n, p} \nE{q(\MU)}{
		\nE{p(\MF_n \mid \MU)}{
			\log p(\MY_{n, p} \mid T_p(\MF_{n, p}))
		}
	} \\
	&= \sum_{n, p} \nE{q(\MU_k)}{
		\nE{p(\MF_n \mid \MU_k)}{
			\log p(\MY_{n, p} \mid T_p(\MF_{n, p}))
		}
	} 
\end{aligned}
\end{equation}
where $k = t(n)$ is the time period associated with data $\MX_n$. Here $\MU_k$ are the spatial inducing points at time $t(n)$ and hence $\MU_k \in \R^{Q M_s}$. We need to compute 
\begin{equation}
\begin{aligned}
	\diff{\ELL}{\MS} &= \sum_{n, p} \diff{}{\MS} \, \nE{q(\MU_k)}{
		\nE{p(\MF_n \mid \MU_k)}{
			\log p(\MY_{n, p} \mid T_p(\MF_{n, p}))
		}
	} \\
	&= \sum_{n, p} \MP_k \cdot \diff{}{\MS_k} \, \nE{q(\MU_k)}{
		\nE{p(\MF_n \mid \MU_k)}{
			\log p(\MY_{n, p} \mid T_p(\MF_{n, p}))
		}
	} \, \cdot \MP^\T_k\\
	&= \sum_{n, p} \MP_k \cdot \nE{q(\MU_k)}{
		\hess{}{\MU_k} \, \nE{p(\MF_n \mid \MU_k)}{
			\log p(\MY_{n, p} \mid T_p(\MF_{n, p}))
		}
	} \, \cdot \MP^\T_k\\
	&\stackrel{\text{delta}}{\approx} \sum_{n, p} \MP_k \cdot \nE{q(\MU_k)}{
		\hess{}{\MU_k} \, \log p(\MY_{n, p} \mid T_p(\MF^{*}_{n, p}))
	} \, \cdot \MP^\T_k\\
	&\stackrel{\text{Gauss-Newton}}{\approx} \sum_{n, p} \MP_k \cdot \nE{q(\MU_k)}{
		\left[\,\diff{T_p(\MF^{*}_{n, p})}{\MU_k}\,\right]^\T \, 
		\hess{\log p(\MY_{n, p} \mid T_p) }{T_p} \, \left[\,\diff{T_p(\MF^{*}_{n, p})}{\MU_k}\,\right]
	} \, \cdot \MP^\T_k \\ 
	&= \sum_{n, p} \MP_k \cdot \nE{q(\MU_k)}{ \MJ_k^\T \, \MH_k \MJ} \, \cdot \MP^\T_k
\end{aligned}
\end{equation}
where the shapes are
\begin{align}
	\MJ &\in Q M_s \times 1	\\
	\MH_k &\in 1 \times 1
\end{align}
and $\MP_k$ is a permutation matrix that permutes from data-latent to latent-data format. In implementation, we do not need to perform this permutation as we only require the blocks $\diff{\ELL}{\MS_k}$ but write it here for completeness.

\subsubsection{Optimality of Natural Gradients In Linear Models} \label{sec:app_cvi_nat_grad_optimal}

\begin{theorem}
	Consider a linear multi-task model of the form
	\begin{equation}
	\begin{aligned}
		f_q(\cdot) &\sim \GP(0, \MK_q) \\
		\StackedF(\vx) &= [\vf_q(\vx)]^\NumLatents_{q=1} \\
		\MY_n &= \MixingMatrix  \, \StackedF(\MX_n) + \psi ~~ \text{where} ~~ \psi \sim \nGauss{\Zero}{\nBlkdiag{[\sigma^2_p]^{\NumTasks}_{p=1}}}
	\end{aligned}
	\end{equation}
	then under a full Gaussian variational approximate posterior
	\begin{equation}
		q(\StackedF) \defeq \nN{\StackedF}{\vm}{\MS}
	\end{equation}
	where $\vm \in \R^{(N \times Q) \times 1}, \MS \in \R^{(N \times Q) \times (N \times Q)}$ then the natural gradient update with a learning rate of $1$ recovers the optimal solution $p(\StackedF \mid \MY)$.
\end{theorem}
\begin{proof}
	To prove this we first derive the natural parameters of the posterior $p(\StackedF \mid \MY)$. We then derive the closed form expression of the natural parameter update and show that they recover that of the posterior. Let 
	\begin{equation}
		\MultiOutputF(\MX) = \MixingMatrixKron \, \StackedF(\MX) \sim \nN{\MultiOutputF(\MX)}{\Zero}{\TransW}
	\end{equation}
	where $\StackedGram{\MX}{\MX} = \nBlkdiag{[\MK_q]^{\NumLatents}_{q=1}}$ and $\MixingMatrixKron = \MixingMatrix \kron \Eye$ then the posterior $p(\MultiOutputF(\MX) \mid \MY)$ is Gaussian 
	\begin{equation}
	\begin{aligned}
		p(\MultiOutputF(\MX) \mid \MY) &= \nN{\MultiOutputF(\MX)}{\mu_{\MultiOutputF \mid \MY}}{\Sigma_{\MultiOutputF \mid \MY}} 
	\end{aligned}
	\end{equation}
	with
	\begin{equation}
	\begin{aligned}
		\mu_{\MultiOutputF \mid \MY} &= \left[ \TransW \right] \, \left[\TransW + \Phi \right]^{-1} \, \MY \\
		\Sigma_{\MultiOutputF \mid \MY} &= \left[\left[ \TransW \right]^{-1} + \Phi^{-1} \right]^{-1}.
	\end{aligned}
	\end{equation}
	The covariance matrix can be simplified by invoking Woodbury's identity twice
	\begin{equation}
	\begin{aligned}
		\Sigma_{\MultiOutputF \mid \MY} &= \left[\TransW\right] - \left[\TransW\right] \, \left[ \Phi +  \left[\TransW\right] \right]^{-1}\left[\TransW\right] \\
		&= \MixingMatrixKron \, \left[ \StackedGram{\MX}{\MX} - \StackedGram{\MX}{\MX} \, \MixingMatrixKron^\T \left[\Phi +  \TransW\right]^{-1} \, \MixingMatrixKron \, \StackedGram{\MX}{\MX} \right] \, \MixingMatrixKron^\T \\
		&= \MixingMatrixKron \, \left[\MixingMatrixKron \, \Phi^{-1} \, \MixingMatrixKron^\T+  \StackedGram{\MX}{\MX}^{-1} \right]^{-1} \, \MixingMatrixKron^\T
	\end{aligned}
	\end{equation}
	and the mean can be expressed as
	\begin{equation}
	\begin{aligned}
		\mu_{\MultiOutputF \mid \MY} &= \Sigma_{\MultiOutputF \mid \MY} \, \Phi^{-1} \MY \\
		&= \MixingMatrixKron \, \left[\MixingMatrixKron \, \Phi^{-1} \, \MixingMatrixKron^\T+  \StackedGram{\MX}{\MX}^{-1} \right]^{-1} \, \MixingMatrixKron^\T \, \Phi^{-1} \MY.
	\end{aligned}
	\end{equation}
	Now we can immediately read off the posterior $p(\StackedF \mid \MY)$ as $p(\MultiOutputF(\MX) \mid \MY) = p(\MixingMatrixKron\,\StackedF \mid \MY)$ is simply a transformed version
	\begin{equation}
	\begin{aligned}
		p(\StackedF \mid \MY) &= \nN{\StackedF}{\left[\MixingMatrixKron \, \Phi^{-1} \, \MixingMatrixKron^\T+  \StackedGram{\MX}{\MX}^{-1} \right]^{-1} \, \MixingMatrixKron^\T \, \Phi^{-1} \MY} {\left[\MixingMatrixKron \, \Phi^{-1} \, \MixingMatrixKron^\T+  \StackedGram{\MX}{\MX}^{-1} \right]^{-1}} \\
		&= \nN{\StackedF}{\mu_{\StackedF \mid \MY}}{\Sigma_{\StackedF \mid \MY}}
	\end{aligned}
	\end{equation}
	whose natural parameters are
	\begin{equation}
		\natp_{\StackedF \mid \MY} = \left[ 
			\MixingMatrixKron^\T \, \Phi^{-1} \MY, 
			-\frac{1}{2} \, \MixingMatrixKron \, \Phi^{-1} \, \MixingMatrixKron^\T -  \frac{1}{2} \, \StackedGram{\MX}{\MX}^{-1}
		\right]^\T.
	\end{equation}
	We now derive the closed form expression of the natural gradient update with a learning rate of $1$, and show that it recovers $\natp_{\StackedF \mid \MY}$. The expected log likelihood  (\ELL) is 
	\begin{equation}
	\begin{aligned}
		\text{\ELL} &= \nE{q(\StackedF)} {\log \nN{\MY}{\MixingMatrixKron \, \StackedF}{\Phi}} \\
		&= \log \nN{\MY}{\MixingMatrixKron \, \StackedF}{\Phi} - \frac{1}{2} \, \nTr{\Phi^{-1} \, \TransS}.
	\end{aligned}
	\end{equation}
	The required derivatives are
	\begin{equation}
	\begin{aligned}
		\diff{\ELL}{\vm} &= - \frac{1}{2} \, \diff{}{\vm} \left[ (\MY - \MixingMatrixKron \, \vm)^\T \, \Phi^{-1} \,  (\MY - \MixingMatrixKron \, \vm) \right] \\
		&=  \MixingMatrixKron^\T \, \Phi^{-1} \,  (\MY - \MixingMatrixKron \, \vm)
	\end{aligned}
	\end{equation}
	where the last follows because $\Phi$ is symmetric and
	\begin{equation}
	\begin{aligned}
		\diff{\ELL}{\MS} &= - \frac{1}{2} \, \diff{}{\MS} \left[ \nTr{\Phi^{-1} \, \TransS} \right] \\
		&= - \frac{1}{2} \, \TransPhiInv.
	\end{aligned}
	\end{equation}
	The natural gradient is now given as
	\begin{equation}
	\begin{aligned}
		\diff{\text{\ELL}}{\meanp_{\StackedF \mid \MY}} &= \left[ \begin{matrix} 
			\diff{\ELL}{\vm} - 2 \, \diff{\ELL}{\MS}^\T \, \vm \\
			\diff{\ELL}{\MS}
		\end{matrix}\right] = \left[ \begin{matrix} 
			\MixingMatrixKron^\T \, \Phi^{-1} \,  (\MY - \MixingMatrixKron \, \vm) - \MixingMatrixKron^\T \, \Phi^{-1} \, \MixingMatrixKron \, \vm \\
 			- \frac{1}{2} \, \TransPhiInv
		 \end{matrix}\right] \\
		 &= \left[ \begin{matrix} 
			\MixingMatrixKron^\T \, \Phi^{-1} \, \MY \\
 			- \frac{1}{2} \, \TransPhiInv
		 \end{matrix}\right].
	\end{aligned}
	\end{equation} 
	The natural gradient update with a learning rate of $1$ is 
	\begin{equation}
		\natp_{q(\StackedF)} =  \diff{\text{\ELL}}{\meanp_{\StackedF \mid \MY}} + \natp_{p(\StackedF)}
	\end{equation}
	where $\natp_{p(\StackedF)} = \left[\Zero, -\frac{1}{2} \MK^{-1}\right]^\T$ are the natural parameters of the prior $p(\StackedF)$, hence after the update the natural parameters are
	\begin{equation}
		\natp_{q(\StackedF)}  = \left[ \begin{matrix} 
			\MixingMatrixKron^\T \, \Phi^{-1} \, \MY \\
 			- \frac{1}{2} \, \TransPhiInv -\frac{1}{2} \MK^{-1}
		 \end{matrix}\right].
	\end{equation}
	which recover those of $p(\StackedF \mid \MY)$, and hence we recover the optimal posterior. 
\end{proof} %
\section{Further Experimental Details and Results}
\label{sec:app_experimental_details}

\EKS methods were run on CPUs. State-space methods running on GPU used the parallel form of the Kalman smoother (see \citep{sarkka_parallel_smoother:2021,hamelijnck_spatio_temporal:2021}).

\subsection{An extension of \AUTOIP} \label{sec:app_autoip_speedup}

If one drops the requirement for state-space representations then the approximations proposed in \cref{sec:reducing_spatial_complexity} directly define approximations to the variational \gp defined by \cref{eqn:general_vi_elbo}, and hence directly extend \AUTOIP. For example on the non-linear damped pendulum in \cref{sec:exp} we run this extension of \AUTOIP with whitening and $50$ inducing points for $C=1000$ and achieve an RMSE of $0.06 \pm 0.001$ and running time of $158.16 \pm 0.34$, clearly improving the running time against \AUTOIP. However the benefit of our methods is that \PIGP remains linear in temporal dimensions which is vital for applications that are highly structured in time \cite{hamelijnck_spatio_temporal:2021}.

\subsection{Modelling Unknown Physics} \label{sec:app_unknown_physics}

Modelling of missing physics can be handled by parameterising unknown terms with \gps. For example take a simple non-linear pendulum
\begin{equation}
	\frac{d^2 \theta}{dt^2} + \sin(\theta) = 0.
\label{eqn_app_lfm_pendulum}
\end{equation}
Now consider that the the $\sin(\theta)$ is unknown and we would like to learn it. If we define the our differential equation in \cref{eqn:phi_gp_general_pde} as
\begin{equation}
	g=\frac{d^2 \, f_1}{dt^2} + f_2(t) = 0
\end{equation}
where both $f_1(\cdot), f_2(\cdot)$ are latent \gps that we wish to learn. We now construct 300 observations for training from the solution of  \cref{eqn_app_lfm_pendulum} across the range $[0, 30]$ and $1000$ for testing. We run \PIGP and compare the similarity of the learnt latent \gp $f_2(\cdot)$ to the true function at the test locations and achieve an \RMSE of $0.068$ indicating we have recovered the latent force/unknown physics well.

\subsection{Monotonic Timeseries} \label{sec:app_exp_monotonic}

This first example showcases the effectiveness of \PIGP in learning monotonic functions. Monotonicity information is expressed by regularising the first derivative to be positive at a set of collocation points \citep{riihimaki_monotonic_gp:2010}:
\begin{equation*}
\begin{aligned}
	p(\MY \mid \vf) &= \nN{\MY}{\vf}{\sigma^2_y}, ~~ 
	p(\Zero \mid \diff{\vf}{t}) &= \Phi(\diff{\vf}{t} \cdot \frac{1}{v}) 
\end{aligned}
\end{equation*} 
where $\Phi(\cdot)$ is a Gaussian cumulative distribution function, and $v=1e-1$ is a tuning parameter that controls the steepness of the step function.  We plot predictive distributions of (batch) \gp and \PIGP in \cref{fig:monotonic_timeseries}. The \gp fits data and does not learn a monotonic function. However, using $300$ collocation points, \PIGP is able to include the additional information and learn a monotonic function whilst running $1.5$ times faster.
\begin{figure}[!h]
	\centering
	\input{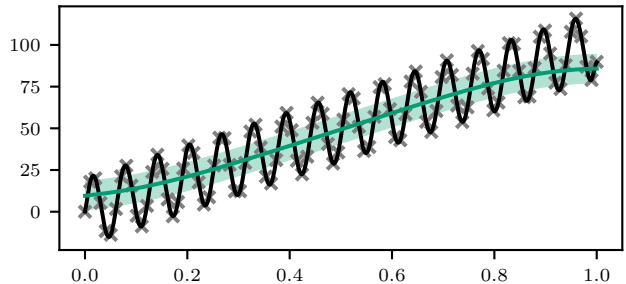}
	\caption{Predictive distributions of \gp and \PIGP on the monotonic function in \cref{sec:app_experimental_details}. The \gp cannot incorporate monotonicity information and fits the data.}
	\label{fig:monotonic_timeseries}
\end{figure}

\subsection{Non-linear Damped Pendulum} \label{sec:app_pendulum}

All models were run using an Nvidia Titan RTX GPU and an Intel Core i5 CPU. All were optimised for $1000$ epochs using Adam \citep{kingma_adam:2014} with a learning rate of $0.01$. Both the \gp and \AUTOIP had an \RBF kernel (following \citet{long_autoip:2022}) and \PIGP used a \MaternSevenTwo; all with a lengthscale of $1.0$. The observation noise was initialised to $0.01$ and the collocation $0.001$. Both were fixed for the first $40\%$ of training and then released. Predictive distribution of \PIGP and \AUTOIP are plotted in \cref{fig:app_exp_pendulum}.
\begin{figure}
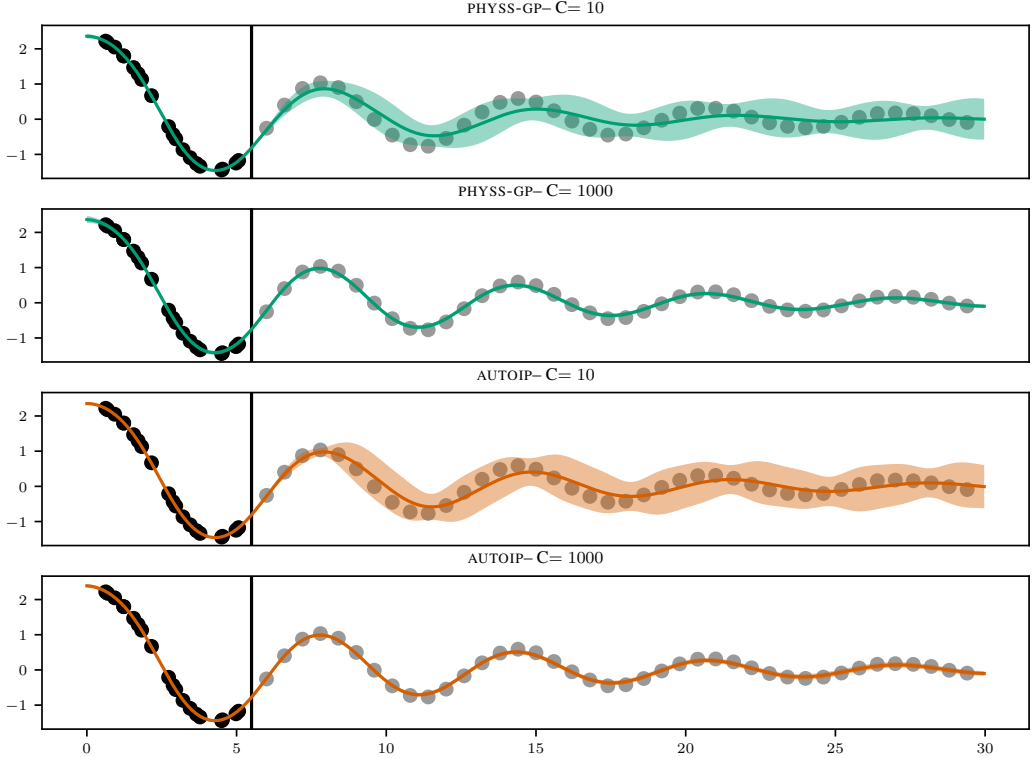

	\resizeinpgf{figs/pendulum}{pendulum.pgf}
	\caption{Predictive distributions on the Damped Pendulum.}
	\label{fig:app_exp_pendulum}
\end{figure}

\subsection{Curl-free Magnetic Field Strength} 

All models were run using an Nvidia Titan RTX GPU and an Intel Core i5 CPU. All models are run for $5000$ epochs using Adam with a learning rate of $0.01$, and use a \MaternThreeTwo kernel on time, with ARD \RBF kernels on the spatial dimensions, with a lengthscale of $0.1$ across all. The Gaussian likelihood is initialised with a variance of $0.01$ and held for $40\%$ of training. All our methods used a natural gradient learning rate of $1.0$ as this is the conjugate setting. 

\subsection{Diffusion-Reaction System} 
We use data provided by \citep{raissi_pinn:2019} under an MIT license. All models were run using an Nvidia Titan RTX GPU and an Intel Core i5 CPU. Our method \sPIGP and \hsPIGP use a \Matern72 kernel on time and an \RBF of space, both initialised with a lengthscale of $0.1$. We place the collocation points on a regular grid of size $20 \times 10$ and use $\Ms=20$ spatial inducing points. We pretrain for $100$ iterations using a natural gradient learning rate of $0.01$ and after use a learning rate $0.1$ for the remaining $19000$ iterations.  \AUTOIP uses a \RBF kernel on both time and space with a lengthscale of $0.1$. We place the collocation points on a regular grid of size $10 \times 10$. All models use Adam with a learning rate $0.001$ and train for a total of $20000$ iterations. 

\subsection{Ocean Currents}
Our method \hsPIGP was run using an Nvidia Titan RTX GPU and an Intel Core i5 CPU. We ran for $10000$ iterations, using Adam with a learning rate of $0.01$. For natural gradients with used a learning rate of $0.1$. We used a \MaternThreeTwo kernel on time and \RBF kernels on both spatial dimensions with lengthscales $[24.0, 1.0, 1.0]$. We used $100$ spatial inducing points and a spatial mini-batch size of $10$.

\end{document}